\documentclass[11pt]{article}

\usepackage[preprint]{acl}

\usepackage{times}
\usepackage{latexsym}

\usepackage[T1]{fontenc}

\usepackage[utf8]{inputenc}

\usepackage{microtype}

\usepackage{inconsolata}

\usepackage{graphicx}

\usepackage{multirow}
\usepackage[table]{xcolor}
\usepackage{wrapfig}
\usepackage{lipsum}
\usepackage{tabularx}
\usepackage{booktabs} 
\usepackage{enumitem}

\usepackage{amsmath, amssymb, amsthm}
\usepackage[most]{tcolorbox}
\usepackage{makecell}
\usepackage{soul}

\definecolor{custom_green}{HTML}{ACE8D7}
\definecolor{custom_green_light2}{HTML}{EBF9F5}
\definecolor{custom_green_light}{HTML}{C8F6E9}
\definecolor{custom_green_dark}{HTML}{34B591}
\definecolor{custom_green_dark2}{HTML}{006246}
\definecolor{custom_pink}{HTML}{FFD2E0}
\definecolor{custom_pink_light2}{HTML}{FFF0F5}
\definecolor{custom_pink_light}{HTML}{FFE9F0}
\definecolor{custom_pink_dark}{HTML}{DE72AA}
\definecolor{custom_pink_dark2}{HTML}{962660}

\definecolor{custom_purple_dark}{HTML}{270069}

\newcommand{\hlgreen}[1]{{\sethlcolor{custom_green_light}\hl{#1}}}
\newcommand{\hlpink}[1]{{\sethlcolor{custom_pink_light}\hl{#1}}}

\theoremstyle{plain}
\newtheorem{theorem}{Theorem}[section]
\newtheorem{proposition}[theorem]{Proposition}

\theoremstyle{definition}

\newtheorem{assumption}[theorem]{Assumption}
\theoremstyle{remark}

\title{Understanding Reasoning in LLMs through \\Strategic Information Allocation under Uncertainty}

\author{
  Jeonghye Kim$^{1,2}$\thanks{Work done during Jeonghye's internship at Microsoft Research. \textsuperscript{$\dagger$} Corresponding author. Our analysis code can be found at \href{https://github.com/beanie00/strategic-information-allocation-llm-reasoning}{\texttt{link}}.}, Xufang Luo$^1$\textsuperscript{$\dagger$}, Minbeom Kim$^3$, \\
  {\bfseries Sangmook Lee$^3$, Dongsheng Li$^1$, Yuqing Yang$^1$} \\[5pt]
  $^1$Microsoft Research\, $^2$KAIST, $^3$Seoul National University
}

\begin{document}
\maketitle
\begin{abstract}

LLMs often exhibit Aha moments such as self-correction after tokens like "Wait," yet the underlying mechanism remains unclear. Standard LLMs collapse mainly through silent divergence, where trajectories drift from the correct answer yet remain locally coherent, so no explicit error triggers reactive self-correction. We introduce an information-theoretic framework that separates reasoning into procedural advancement and \emph{epistemic verbalization}, the token-level externalization of uncertainty, and prove that sporadic verbalization restores convergence toward the correct answer even without explicit error triggers. Empirically, a minimal doubt cue recovers failed trajectories, and small-scale SFT suffices to instill or suppress this capability, suggesting that strong reasoning hinges less on an extraordinary inner mechanism than on the linguistic habit of externalizing uncertainty. Our framework recasts reasoning as strategic information allocation under uncertainty, offering a new lens for understanding and advancing LLM reasoning.

\end{abstract}

\addtocontents{toc}{\vskip 0.5em}
\addtocontents{toc}{\protect\contentsline{section}{[Main]}{}{}} 

\section{Introduction}

Recent large language models (LLMs) often exhibit so-called Aha moments during reasoning, with behaviors such as self-correction or reflection appearing after tokens like "Wait" \citep{deepseek, understanding_aha}. These phenomena are frequently cited as key mechanisms underlying effective reasoning, yet there remains little consensus on what computational or informational role they actually play \citep{d2026illusion, maynot_aha, tsui2025self}. Prior work tends to group Aha moments, reflection, self-correction, and the emergence of specific tokens as a single class of phenomena, making it difficult to disentangle their underlying mechanisms.

\begin{figure}
    \centering
    \includegraphics[width=\linewidth]{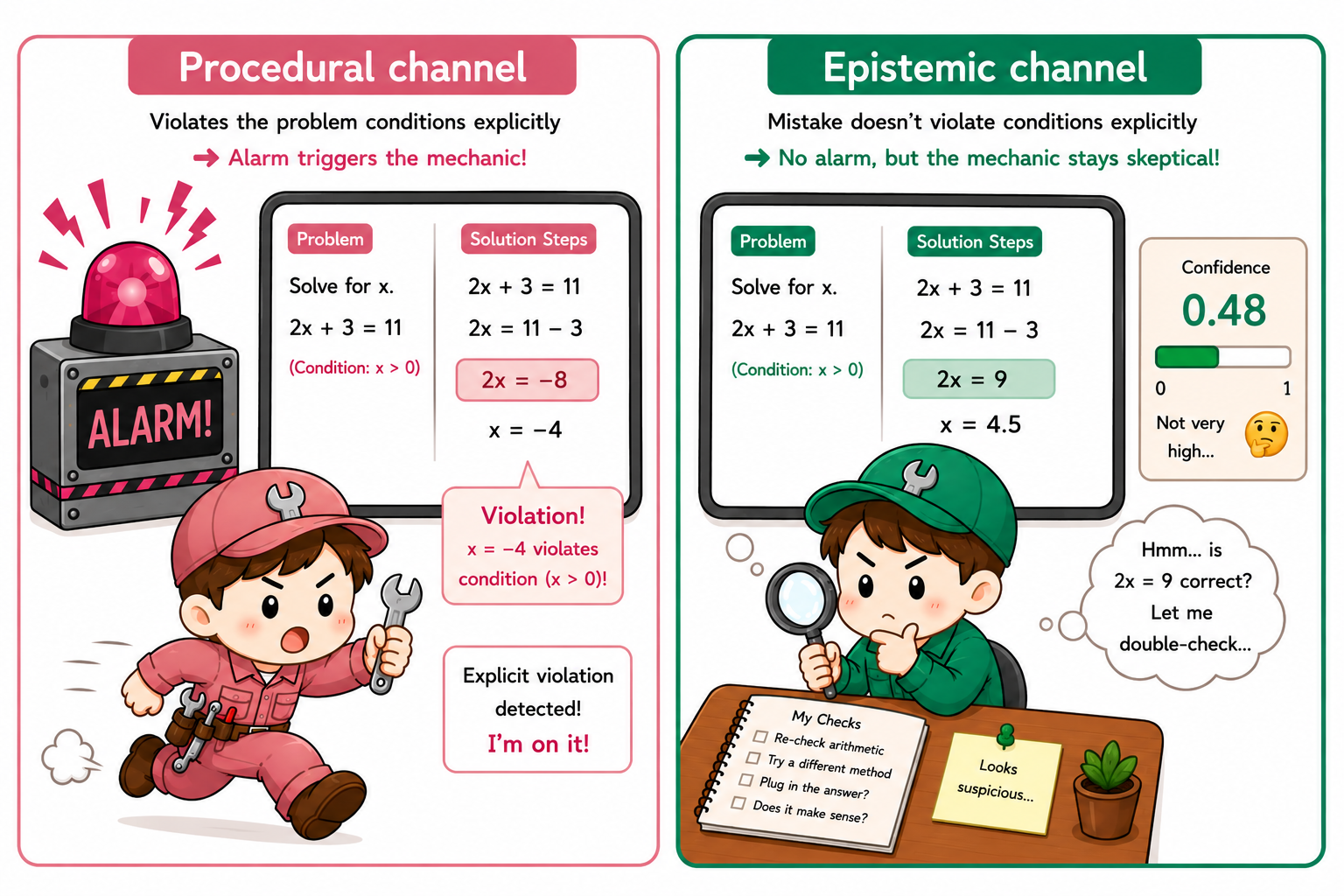}
    \caption{\textbf{Two pathways for self-correction.}
The Procedural channel reacts only when an explicit violation triggers an alarm, whereas the Epistemic channel continuously monitors and prompts self-checking even without any visible contradiction.}
    \label{fig:concept}
\end{figure}

In parallel, recent studies have examined reasoning from an information-theoretic perspective \citep{understanding_cot, liang2025chain}, reinterpreting Chain-of-Thought (CoT) \citep{cot} as a process of information accumulation toward the correct answer. While offering valuable insights, these approaches largely assume procedural, step-by-step execution and do not fully account for the self-corrective behaviors of modern reasoning models, particularly recovery after entering an incorrect trajectory. Once execution enters an erroneous path, reasoning may remain locally coherent yet globally incorrect, with the underlying error unrecognized.

To investigate this gap, we analyze the reasoning behavior of standard LLMs and large \textit{reasoning} models (LRMs). LLMs rely almost exclusively on \emph{\textbf{reactive correction}}, revising their trajectory only when an explicit contradiction or failed check surfaces. When errors remain latent, no trigger arises and the trace drifts into reasoning collapse, the dominant failure mode. LRMs additionally exhibit \emph{\textbf{proactive correction}}, questioning prior steps without any overt error and surfacing hidden mistakes that reactive mechanisms cannot reach (Figure \ref{fig:concept}). Although such proactive signals are imprecise and often second-guess already-correct chains, in this regime even a noisy signal is more useful than a precise one that never fires.

Building on this observation, we identify an additional informational axis distinct from procedural information. Our key idea is \emph{\textbf{epistemic uncertainty verbalization}}, the explicit externalization at the token level of a model's internal uncertainty about its reasoning state. Since autoregressive generation conditions only on preceding tokens, latent uncertainty is informationally inert, whereas verbalized uncertainty becomes actionable and supports downstream control actions such as self-correction. From this perspective, commonly discussed tokens such as "Wait" are surface manifestations of epistemic verbalization rather than the mechanism itself, and reasoning can be reframed as \emph{strategic information allocation under uncertainty}: the joint determination of how to advance the procedure and how much epistemic verbalization to allocate at each step, given the model's current uncertainty.

We empirically validate that epistemic verbalization carries substantive information for reasoning. On failed trajectories, injecting a minimal doubt cue such as "Wait, is that correct?", without identifying what went wrong, recovers around 15\% of incorrect rollouts, showing that the externalization of uncertainty itself supplies actionable signal. We further show that epistemic verbalization is a rapidly learnable capability. In line with the findings of \citet{limo}, as few as 800 training samples suffice to either instill or suppress this behavior, indicating that externalizing internal uncertainty requires no extraordinary ability but rather reflects a linguistic habit. Our new taxonomy and framing of LLM reasoning offers a unified lens for interpreting Aha moments, self-correction, and post-training phenomena, and points toward principled directions for designing and controlling reasoning models.

\section{Related Works} \label{section:related_work}

\paragraph{Understanding Aha moments.}
Recent studies have cast doubt on Aha-like phenomena in LLM reasoning. \citet{d2026illusion} show that markers such as ``Wait'' emerge from high-entropy prediction states and correlate only weakly with performance gains. \citet{maynot_aha} find that apparent self-reflection often degenerates into repetition rather than producing constructive revisions. A related line of evidence shows that LLMs correct externally provided errors yet fail to fix the same errors in their own outputs \citep{tsui2025self, cannot_correct, tyen2024llms, kamoi2024evaluating}. These studies document when Aha tokens fail to predict correction, but leave open why: the marker itself may be unreliable, the model may lack a corrective policy, or both. The underlying mechanism remains unidentified.

\paragraph{Theoretical understanding of reasoning.}
A parallel line of work analyzes reasoning trajectories at the aggregate level. \citet{yang2025decoupling} show that reasoning can both fix and introduce errors by decoupling knowledge-based responses from reasoning-based corrections. \citet{gan2025cot} recast CoT as optimization over reasoning states and identify trade-offs between noise reduction and generalization. Information-theoretic analyses show that CoT preserves task-relevant information and tightens error bounds \citep{understanding_cot, liang2025chain}. \citet{mi_peak} observe that effective reasoning concentrates information gain in a few critical steps, often marked by tokens like ``Wait.'' These accounts describe how information flows across a trajectory but treat every token as a procedural step, leaving unexplained how a model recovers after drifting onto an incorrect path without outside help.

We close both gaps with a unified information-theoretic framework of reasoning at the token level. First, we identify \emph{epistemic verbalization}, the externalization of uncertainty about one's own reasoning, as the informational source that lets a model regain traction after a procedural misstep, and prove that even sporadic occurrences suffice to restore convergence. Second, we separate this informational role from the \emph{control actions}, such as self-correction, that act on it. Together, these results reconcile the mixed empirical findings on Aha moments and supply the token-level grounding missing from aggregate information-theoretic accounts.

\section{A Self-Conditioning Framework for LLM Reasoning}
\label{sec:self_baysian_reasoning}
Our analysis mainly focuses on the \emph{closed-world inference setting}, where an LLM operates without external observations at inference time. Unlike tool-augmented agents that reduce uncertainty through environmental interaction, a closed-world LLM is constrained to a fixed parameterization $\theta$ and an initial input $x$, so all progress toward correct inference must come from internal belief transformation. We formalize this as \textbf{self-Bayesian reasoning}, where CoT~\citep{chain_of_thought} acts as \emph{self-conditioning} that reshapes the model's belief over the target variable $Y$ without new external observations. Appendix~\ref{appendix:world_bayesian} extends this to the open-world setting. 

\subsection{Reasoning as Self-Conditioning}
\label{sec:self_conditioning}
Given an input $x \in \mathcal{X}$, an LLM parameterized by $\theta$ induces a conditional distribution $P_\theta(\cdot \mid x)$ over tokens from vocabulary $\mathcal{V}$, and autoregressively generates a sequence $a_1, \dots, a_T$ until terminal step $T$. We define the reasoning state at step $t$ as
\begin{equation}
    s_0 := x, \qquad s_t := (x, a_1, \dots, a_t) \text{ for } t \geq 1,
\end{equation}
so that each token is drawn as $a_t \sim P_\theta(\cdot \mid s_{t-1})$. Each state $s_t$ induces a predictive distribution $P_\theta(Y \mid s_t)$ over the target variable $Y \in \mathcal{Y}$, and we take the objective of reasoning to be the minimization of $H(Y \mid s_T)$, where $H(\cdot)$ denotes Shannon entropy. Reasoning thus proceeds by \emph{self-conditioning}: each generated token refines the model's belief over $Y$ without introducing external evidence, and the information gain of a step is
\begin{equation}
    \Delta_t := H(Y \mid s_{t-1}) - H(Y \mid s_t).
\end{equation}
Self-generated intermediate steps serve as the model's sole source of evidence, making their informational quality critical for reasoning success.

\subsection{Procedural Reasoning and Its Collapse}
\label{sec:procedural_limit}
A dominant class of self-generated evidence in LLM reasoning consists
of \textbf{procedural information}: explicit step-by-step
computations, symbolic manipulations, variable instantiations, and
executions of learned subroutines~\citep{step_dpo, step_by_step_math,
raise, understanding_cot}. Formally, let
$0 = t_0 < t_1 < \cdots < t_K = T$ partition a reasoning trace into
sub-tasks with task-level state $U_k := (x, a_1, \dots, a_{t_k})$.
Procedural reasoning is then a sequence
$U_k = \Lambda_\theta(U_{k-1}, \tau_k)$, where $\Lambda_\theta$ is an
autoregressive execution operator implementing sub-task $\tau_k$.

\paragraph{Reasoning collapse in standard LLMs.}
We examine how this regime fails in practice by analyzing reasoning
traces from the Qwen2.5 and Qwen3-Base series (Qwen2.5-7B,
Qwen2.5-32B-Instruct, Qwen2.5-Math-7B~\citep{qwen2.5},
Qwen3-1.7B/8B/14B-Base~\citep{qwen3}) on AIME24/25, AMC23, and
MATH500. These models exhibit strong step-by-step procedural reasoning ability, achieve competitive performance on these benchmarks, and serve as standard
baselines in recent math reasoning research~\citep{limo, limr, 80_20, understanding_aha}. Using GPT-5 as an automated judge on 4,800 generated samples (600 problems × 8 samples), we quantify the prevalence of five recurring modes of reasoning collapse across the traces (Figure~\ref{fig:reasoning_collapse}): incoherence, hallucination spirals, repetition, topic drift, and degenerate loops. Collapse is
pervasive, occurring in 50--83\% of incorrect responses and rising
sharply with problem difficulty
(Figure~\ref{fig:reasoning_collapse}a). See
Appendix~\ref{appendix:reasoning_collapse} for details.

\begin{figure}[h!]
    \centering
    \includegraphics[width=\linewidth]{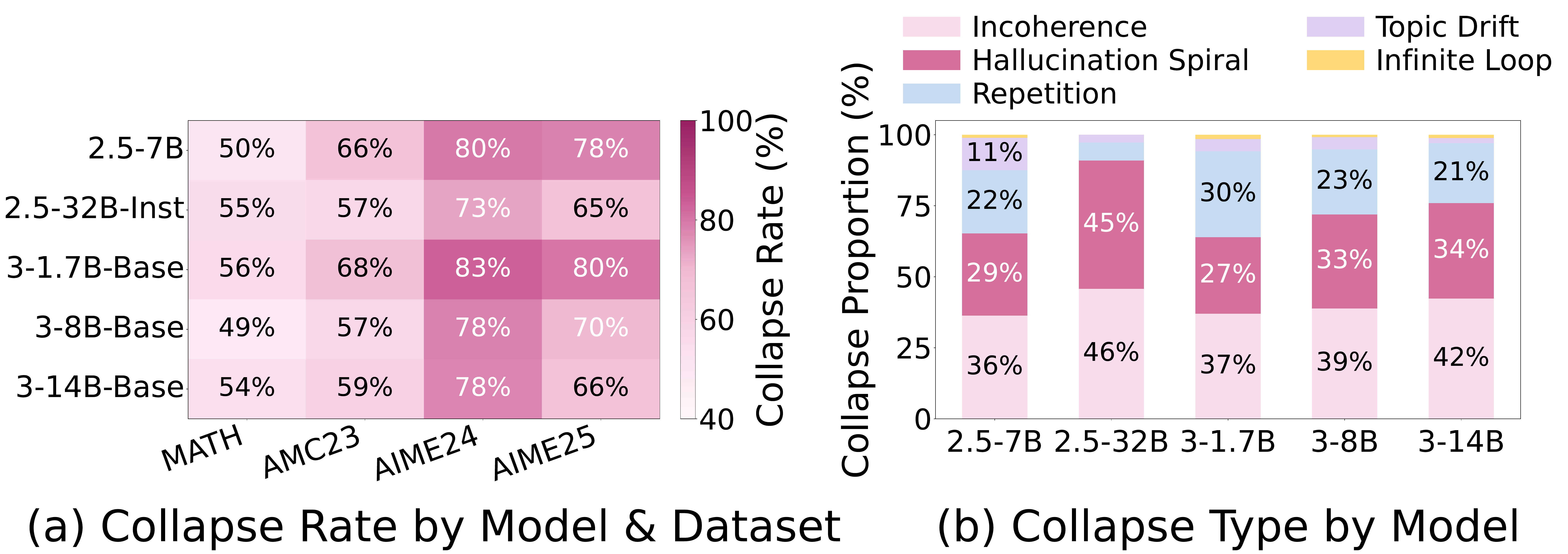}
   \caption{Reasoning collapse analysis across models and benchmarks. 
    (a) Collapse rate among incorrect responses per model and dataset. 
    (b) Breakdown of collapse types per model. Model names omit the ``Qwen'' prefix (e.g., 2.5-7B refers to Qwen2.5-7B).}
    \label{fig:reasoning_collapse}
\end{figure}

In each mode, the surface structure of step-by-step execution is preserved while $P_\theta(Y \mid s_t)$ drifts from the correct $Y$, often without overt error. This silent divergence is the dominant failure pattern, consistent with prior observations that procedural reasoning is brittle when a sub-task is \emph{unidentifiable}~\citep{understanding_cot} or silently misinstantiated through an early misjudgment.

\subsection{How Models Escape (or Fail to Escape) Collapse} \label{sec:collapse_escape}
We next ask how models escape, or fail to escape, these regimes. For
each self-correction event in the traces above, together with traces from
Large Reasoning Models (LRMs: DeepSeek-R1-Distill-Qwen-7B/32B
\citep{deepseek}, Qwen3-8B/14B \citep{qwen3}), we distinguish two
modes:
\begin{itemize}[itemsep=0pt, topsep=0pt, parsep=0pt, partopsep=0pt, leftmargin=*]
    \item \textit{\textbf{Reactive correction.}} The reasoning chain contains an
    explicit error, such as a contradiction, a failed check, or an
    invalid derivation, that the model recognizes and revises.
    \item \textit{\textbf{Proactive correction.}} No overt error has surfaced, yet the model questions a prior step and verifies or backtracks on internal doubt alone.
\end{itemize}

\begin{table}[h!]
\centering
\small
\setlength{\tabcolsep}{4pt}
\begin{tabular}{llcc}
\toprule
\small
Type & Model & Reactive & Proactive \\
 & & $n$ (\%) & $n$ (\%) \\
\midrule
\multirow{6}{*}{LLM}
  & Qwen3-1.7B-Base$^\dagger$   &  12 (100.0) &   0 (0.0)  \\
  & Qwen3-8B-Base$^\dagger$     &  19 (86.4)  &   3 (13.6) \\
  & Qwen3-14B-Base$^\dagger$    &  23 (92.0)  &   2 (8.0)  \\
  & Qwen2.5-7B$^\dagger$        &  11 (100.0) &   0 (0.0)  \\
  & Qwen2.5-Math-7B$^\dagger$   &  35 (79.5)  &   9 (20.5) \\
  & Qwen2.5-32B-Instruct$^\dagger$    &   8 (80.0)  &   2 (20.0) \\
\midrule
\multirow{4}{*}{LRM}
  & Qwen3-8B                    & 576 (76.3)  & 179 (23.7) \\
  & Qwen3-14B                   & 484 (77.8)  & 138 (22.2) \\
  & DeepSeek-R1-Distill-7B          & 475 (65.0)  & 256 (35.0) \\
  & DeepSeek-R1-Distill-32B         & 471 (69.3)  & 209 (30.7) \\
\bottomrule
\end{tabular}
\caption{Reactive and proactive self-correction counts across models. For each generation, a GPT-5-based filter flags traces with a \emph{wrong path $\rightarrow$ correction $\rightarrow$ correct answer} pattern; $n$
is the number of such traces per model. Proactive correction is
largely absent in LLMs and emerges consistently in LRMs.
$^\dagger$Small sample ($N < 50$). Details can be found in Appendix \ref{appendix:reactive_proactive_correction}.}
\label{tab:sc_breakdown}
\end{table}

Examples of both types of modes are provided in Appendix~\ref{appendix:example_two_correction_mode}. In LLMs, self-correction occurs in at most 35 of 4,800 generations (under 1\%) and is overwhelmingly reactive (Table~\ref{tab:sc_breakdown}). Reactive correction is the dominant escape mechanism, but it only works when errors surface as overt local mistakes. The collapse modes of
Figure~\ref{fig:reasoning_collapse} produce no such signal, so
reactive correction has nothing to trigger on, and the trace
silently diverges, ultimately leading to collapse.

\subsection{Proactive Correction in LRMs}
\label{sec:proactive}

LRMs exhibit a markedly different pattern. As shown in
Table~\ref{tab:sc_breakdown}, proactive corrections account for 22--35\%
of all self-corrections. These corrections arise without any overt error: the model pauses, questions a prior step, and either verifies or backtracks based on internal doubt alone, as shown in Appendix~\ref{appendix:example_two_correction_mode}.
Figure~\ref{fig:capability_map} situates this capability in the broader model landscape. Standard LLMs cluster along the bottom of the plot regardless of their direct fluency, whereas reasoning-tuned models occupy a distinct upper-right region where higher direct fluency is paired with a qualitatively new capacity for proactive correction.

\begin{figure}[h!]
    \centering
    \includegraphics[width=\linewidth]{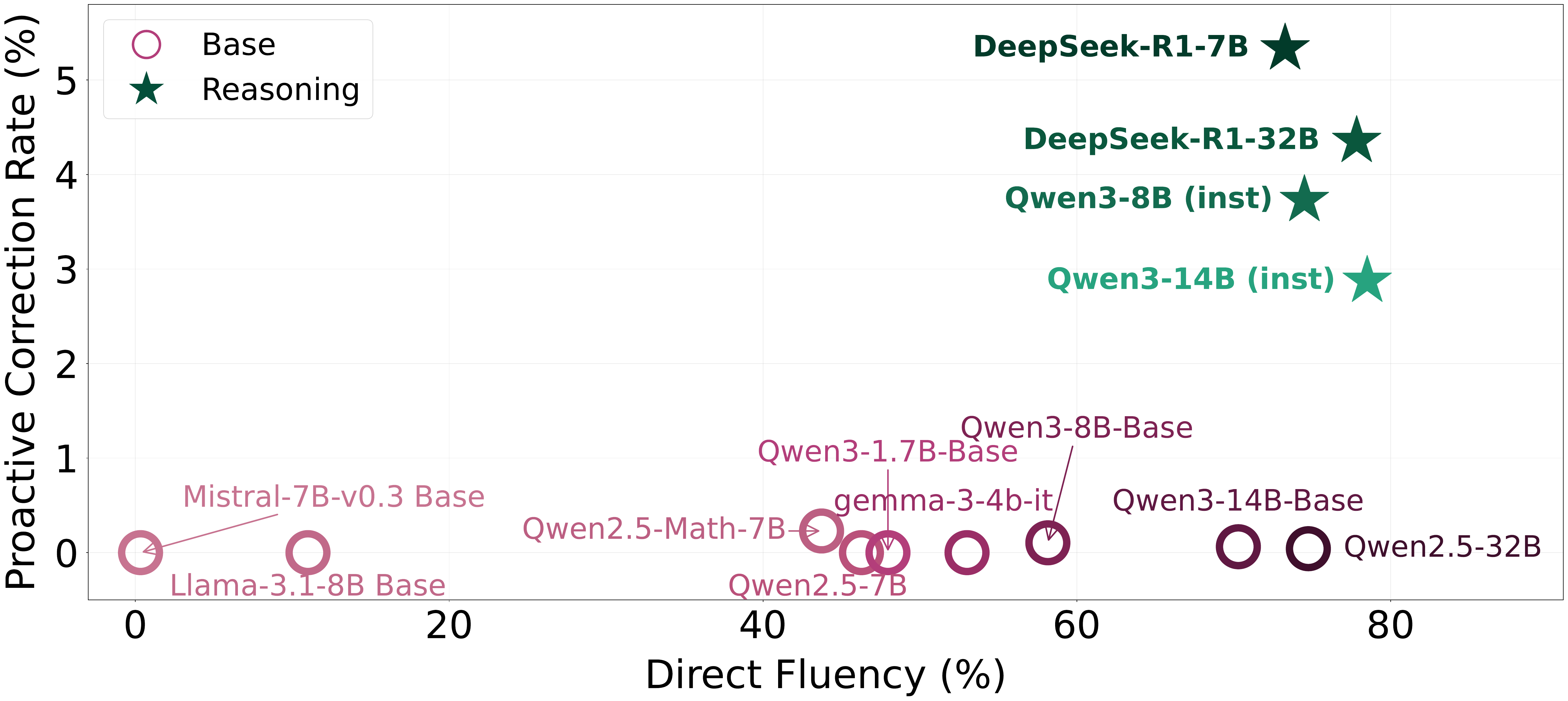}
    \caption{\textbf{Capability map across
model families.} The $x$-axis reports \emph{direct fluency}, the
share of responses answered correctly without invoking any
self-correction, computed as $(\text{correct}-\text{SC})/N$. The $y$-axis reports
the \emph{proactive correction rate}, the share of responses
containing a suspicion-based self-correction, i.e., a correction
triggered without explicit evidence of error.}
    \label{fig:capability_map}
\end{figure}

\paragraph{Precision of proactive signals.} Since proactive correction is preventive, the model may flag and re-check something that turns out to be fine. In other words, prevalence does not imply accuracy. Table~\ref{tab:proactive_precision} reports the \emph{precision} of proactive signals, defined as the fraction of suspicion-triggered corrections issued on a trajectory that was in fact wrong. Across LRMs, precision averages only $24.4\%$: roughly three out of four proactive verbalizations second-guess an
already-correct chain. DeepSeek-R1 distillations are better calibrated
($37.3\%$, $23.9\%$) than the Qwen3 series ($20.5\%$, $15.8\%$), and
precision does not improve with scale.

\begin{table}[h]
    \centering
    \small
    \begin{tabular}{p{0.64\columnwidth} >{\centering\arraybackslash}p{0.25\columnwidth}}
        \toprule
        Model & Precision (\%) \\
        \midrule
        DeepSeek-R1-Distill-Qwen-7B  & 37.3 \\
        DeepSeek-R1-Distill-Qwen-32B & 23.9 \\
        Qwen3-8B                     & 20.5 \\
        Qwen3-14B                    & 15.8 \\
        \bottomrule
    \end{tabular}
    \caption{Precision of proactive self-correction signals across LRMs. Precision is the fraction of suspicion-based corrections issued on genuinely incorrect trajectories, as judged by GPT-5 (See Appendix \ref{appendix:reactive_proactive_correction} for details). The remainder second-guess already-correct chains.}
    \label{tab:proactive_precision}
\end{table}

However, low precision does not undermine the role of proactive
correction. The silent divergences in
Figure~\ref{fig:reasoning_collapse} never produce an overt error, so
reactive correction has no trigger to act on. In this regime, even a
noisy extra signal is more useful than a precise reactive one
that never fires. Proactive correction is what brings these hidden
errors to the surface and lets the model recover, allowing strong
reasoners to maintain trajectory quality where reactive correction
cannot reach. We now term this mechanism \emph{\textbf{epistemic verbalization}}.

\section{Epistemic Verbalization}
\label{sec:epistemic_verbalization}

\paragraph{Definition.}

Epistemic verbalization is the linguistic channel through which the model explicitly externalizes uncertainty about its own trajectory, turning otherwise inert internal assessments into conditionable tokens that subsequent reasoning can act on. Reasoning traces of strong models contain utterances that do not advance procedural execution but instead express judgments of uncertainty or doubt about intermediate steps (Table~\ref{table:ev_examples}).

\begin{table}[h!]
    \centering
    \small
    \begin{tabularx}{\columnwidth}{c X}
        \toprule
        \# & Example \\
        \midrule
        1 & \small{Wait, is that correct?} \\
        2 & \small{Wait, 2023 is 7 multiplied by 17 squared, right?} \\
        3 & \small{Wait, maybe $f(n)$ is related to the Möbius function but scaled differently.} \\
        4 & \small{Wait, perhaps I can write it as $(f * (n/d))(n)=1$, but that doesn't seem helpful.} \\
        \bottomrule
    \end{tabularx}
    \caption{Examples of epistemic verbalizations co-occurring with ``Wait'' in reasoning traces from DeepSeek-R1-Distill-Qwen-1.5B/7B/32B.}
    \label{table:ev_examples}
\end{table}

Let $Z_t$ denote a latent epistemic variable at step $t$, representing 
the model's internal assessment of its trajectory's reliability. As a 
latent quantity, $Z_t$ is informationally inert: although 
$I(Y; Z_t \mid s_{t-1})$ may be strictly positive, this information 
remains trapped inside the model's internal states. Since autoregressive 
decoding conditions on the realized token sequence $s_{t-1}$, not 
on internal states, $Z_t$ cannot influence subsequent generation. 
Epistemic verbalization closes this gap by emitting tokens $v_t$ from 
which $Z_t$ is recoverable. Once $v_t \in s_t$, we have
\begin{equation}
    H(Y \mid s_{t-1}, v_t) \leq H(Y \mid s_{t-1}),
\end{equation}
with strict inequality whenever $v_t$ carries non-trivial information 
about $Z_t$. Epistemic verbalization is thus the mechanism by which 
latent assessments become causally active in reasoning.

\subsection{Injecting Epistemic Verbalization Recovers Failed Trajectories.}  \label{sec:injecting_epistemic}

We empirically validate that epistemic verbalization supplies actionable information beyond procedural reasoning through an intervention study on failed reasoning trajectories.

We collect all incorrect rollouts from Qwen3-8B/14B-Base, which predominantly exhibit procedural reasoning (Section~\ref{sec:procedural_limit}), on three benchmarks (AIME24, AMC23, MATH500) with 8 rollouts per problem. Since every rollout is incorrect, it must diverge from a valid path at some unknown point. We therefore sweep candidate truncation points: for each rollout, we truncate at relative position $\alpha \in \{0.2, 0.5, 0.8, 0.9\}$ and resume generation from that point, optionally inserting a short epistemic verbalization.

\begin{table}[h!]
\small
\centering
\begin{tabular}{p{0.3\columnwidth} p{0.59\columnwidth}}
\toprule
Condition & Phrase \\
\midrule
\textsc{None}              & (no injection) \\
\midrule
\textsc{Hmm-Doubt}         & ``Hmm, I'm not sure this is right.'' \\
\textsc{Hmm-Doubt-Revisit} & ``Hmm, I'm not sure this is right. Let me reconsider.'' \\
\midrule
\textsc{Wait-Doubt}        & ``Wait, is that correct?'' \\
\textsc{Wait-Revisit}      & ``Wait, let me double-check.'' \\
\bottomrule
\end{tabular}
\caption{Injection conditions. Four epistemic phrases span two phrase 
types (\textsc{Hmm}, \textsc{Wait}) and vary in whether they express 
doubt alone, signal an intent to revisit, or both. \textsc{None} 
controls for the effect of mere re-sampling from the truncated prefix.}
\label{tab:injection_conditions}
\end{table}

We compare five conditions (Table~\ref{tab:injection_conditions}). The \textsc{None} baseline performs no injection. The other four are epistemic phrases that vary along two axes. The first axis is phrase type (\textsc{Hmm} vs.\ \textsc{Wait}), and the second is function. \textsc{Doubt} expresses uncertainty alone, \textsc{Revisit} signals an intent to re-examine the reasoning, and \textsc{Doubt-Revisit} combines the two. Because it is unknown whether the truncation point occurs before or after the divergence point, the \textsc{None} condition isolates the effect of re-sampling alone. For each condition we
draw $n=4$ continuations and report the \emph{recovery rate}: the
fraction of originally incorrect rollouts whose continuation reaches the
correct answer.

\begin{figure}[h!]
    \centering
    \includegraphics[width=\linewidth]{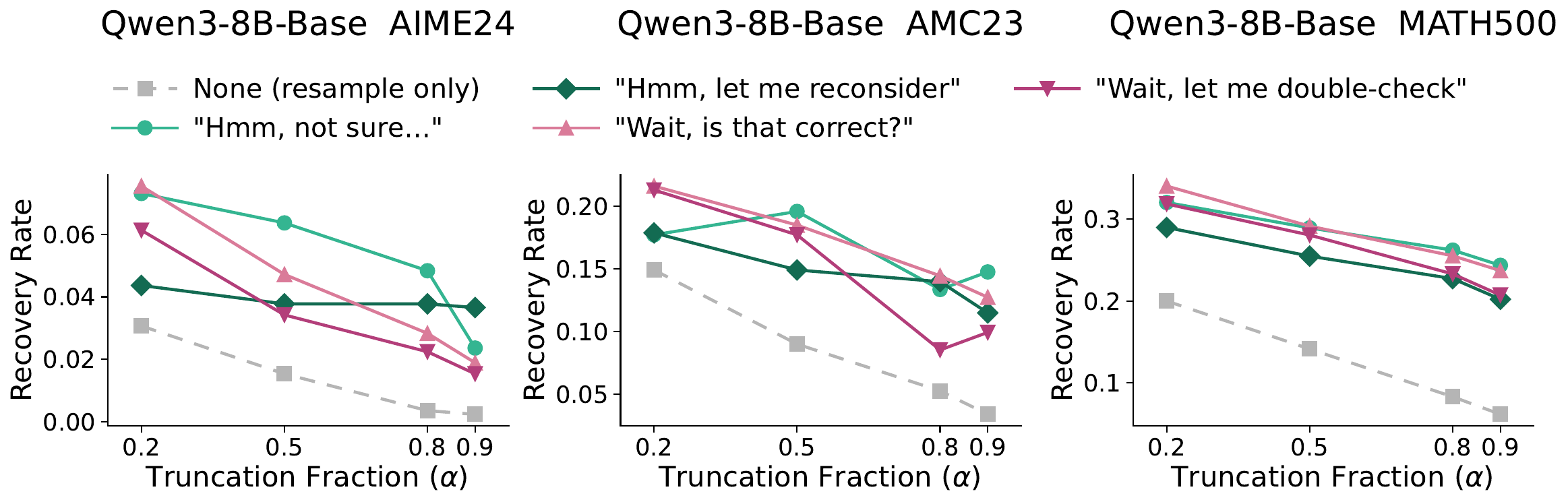}\\[1em]
    \includegraphics[width=\linewidth]{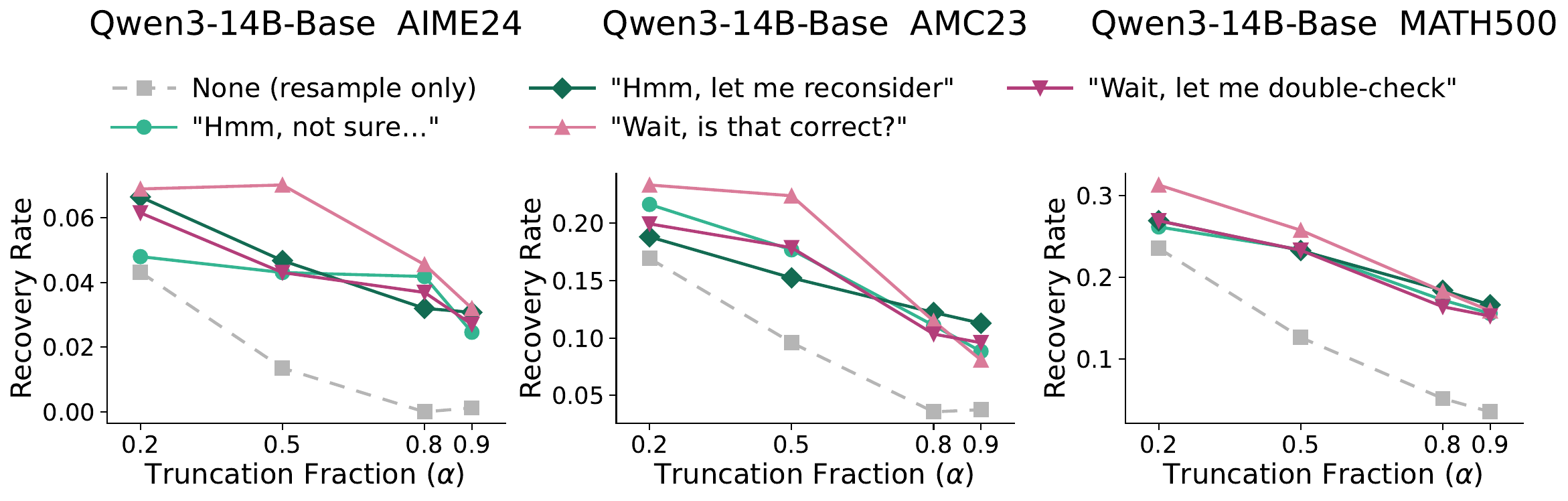}
    \caption{Recovery rate vs.\ truncation fraction $\alpha$ across injection
conditions and benchmarks, for Qwen3-8B-Base (top) and Qwen3-14B-Base (bottom).}
    \label{fig:recovery_rate}
\end{figure}

Figure~\ref{fig:recovery_rate} reports recovery rates across conditions, truncation positions, and benchmarks. Across all settings, simply expressing uncertainty, without pinpointing what went wrong, acts as an informative signal that drives recovery. Notably, even \textsc{Doubt}, which conveys only uncertainty without any explicit cue to re-examine the reasoning (e.g., "Hmm, I'm not sure ..."), suffices for the model to detect its own errors and recover the correct answer. Recovery rates decline as $\alpha$ increases, but the baseline decays fastest, dropping to near zero at $\alpha=0.9$, so the gap between injected cues and the baseline widens with $\alpha$. Differences among phrasings are small, though "Wait"-style cues slightly outperform "Hmm"-style ones. Overall recovery tracks benchmark difficulty (MATH500 $>$ AMC23 $>$ AIME24), and the same pattern holds for both Qwen3-8B-Base and Qwen3-14B-Base, indicating robustness to model scale.

\subsection{Epistemic Tokens as Measurement Proxies} \label{sec:epistemic_tokens}
The linguistic realizations of uncertainty expression can be numerous and diverse, making full semantic analysis infeasible at scale. Still, surface tokens provide a useful observational handle on where verbalization is likely occurring.

To identify which surface tokens reliably co-occur with epistemic verbalization, we analyzed reasoning traces from four models (Qwen3-8/14B and DeepSeek-R1-Distill-Qwen-7B/32B) on math benchmarks and measured how frequently each candidate token appeared using GPT-5 as a judge. Details are provided in Appendix~\ref{appendix:identifying_epistemic_tokens}. The resulting frequencies were: "wait" (73.0\%), "maybe" (32.9\%), "actually" (12.4\%), "check" (10.5\%), "hmm" (8.3\%), "perhaps" (8.2\%), "might" (6.6\%), "seems" (3.3\%), and "alternatively" (1.2\%). 

We therefore adopt these nine terms as \textbf{epistemic tokens}: \textit{"wait"}, \textit{"hmm"}, \textit{"perhaps"}, \textit{"maybe"}, \textit{"actually"}, \textit{"alternatively"}, \textit{"seems"}, \textit{"might"}, and \textit{"check"} as practical indicators of regions where epistemic verbalization likely occurs. 

\section{A Unified Framework: Reasoning as Strategic Information Allocation}
\label{sec:theory}
We now formalize the analysis of Sections \ref{sec:self_baysian_reasoning}–\ref{sec:epistemic_verbalization} within a unified information-theoretic framework: reasoning under uncertainty operates along two informational axes, procedural advancement and epistemic verbalization, on which control actions such as self-correction subsequently act.

\subsection{Epistemic Verbalization Enables Continued Information Acquisition}
\label{subsec:epistemic}
Procedural reasoning operates by proceeding along the correct path or revising an incorrect one when explicit errors are detected. Epistemic verbalization, by contrast, exposes information about the reliability 
of the current trajectory regardless of whether any error has surfaced, thereby 
altering the model's effective belief state. We formalize this distinction by 
specifying when procedural reasoning alone becomes insufficient and showing 
that even sporadic epistemic verbalization suffices to restore convergence.
\begin{assumption}[Procedural Information is Trigger-Conditional]
\label{assump:proc_trigger}
Suppose the procedural trajectory enters a diverged execution path at sub-task index $k^*$. For $k > k^*$, let $E_k \in \{0,1\}$ indicate whether sub-task $\tau_k$ surfaces an overt error available to procedural revision, and let
$p_E := \mathbb{P}(E_k = 1 \mid U_{k-1})$.
There exists $\delta_{\mathrm{proc}} > 0$ such that
\[
I(Y; U_k \mid U_{k-1}) \;\ge\; \delta_{\mathrm{proc}} 
\quad \text{only if } E_k = 1.
\]
\end{assumption}
Procedural information about $Y$ thus accumulates only when overt errors surface along the diverged path. In the silent-divergence regime characterized in 
Section~\ref{sec:procedural_limit}, the surface structure of step-by-step 
execution is preserved even as $P_\theta(Y \mid U_k)$ drifts from the correct 
$Y$, so $p_E \approx 0$ and procedural information acquisition stalls in 
expectation.

\begin{assumption}[Epistemic Informativeness]
\label{assump:epi_info}
There exist $\delta > 0$ and a threshold $\eta > 0$ such that, at every step
$t$ where an epistemic token is generated with $H(Y \mid \tilde S_{t-1}) > \eta$,
\[
H(Y \mid \tilde S_{t-1}) - H(Y \mid \tilde S_t) \;\ge\; \delta,
\]
independently of whether an overt error has surfaced at the current sub-task.
\end{assumption}
Intuitively, whenever uncertainty about $Y$ is still substantial, generating an epistemic token yields a non-trivial reduction in that uncertainty, regardless of whether the procedural channel has flagged an error. The key property is thus \emph{trigger-independence}: epistemic verbalization can introduce information gain even when $p_E \to 0$. This is consistent with the intervention 
study in Section~\ref{sec:injecting_epistemic}, where injecting a doubt cue alone, without pinpointing any error, meaningfully enabled the recovery of failed trajectories. In addition, we empirically demonstrate in Appendix~\ref{appendix:epistemic_verbalization_informativeness} that epistemic verbalization introduces information by analyzing the mutual information between epistemic verbalization and $Y$. 

\begin{proposition}[Sporadic Epistemic Verbalization Suffices for Convergence]
\label{prop:epi_enables_sparse}
Under Assumption~\ref{assump:epi_info}, if an epistemic token is generated with at least some non-zero probability $\rho$ whenever $H(Y \mid \tilde S_t) > \eta$, then $\mathbb{E}[H(Y \mid \tilde{S}_t)] \to 0$ as $t \to \infty$, regardless of the procedural trigger probability $p_E$.
\end{proposition}
The formal statement and proof are deferred to Appendix~\ref{appendix:proof_convergence}. This explains the empirical pattern in Section~\ref{sec:self_baysian_reasoning}. LLMs rely on reactive correction and thus fail in silent-divergence regimes where no procedural trigger arises, whereas LRMs sustain a non-trivial rate of epistemic verbalization that keeps belief contracting in the same regimes.

\subsection{Self-Correction as a Control Action}
\label{subsec:self_correction}
Building on the distinction between procedural and epistemic information, we now separate \emph{information} from \emph{control}. Self-correction is a control action that fires whenever an error surfaces in the reasoning chain. Epistemic verbalization does not itself perform correction. It is a conditionable informational mechanism that renders the model's uncertainty at the token level, enabling the policy to control the reasoning trajectory accordingly.

Taken together, these results characterize reasoning as \emph{\textbf{strategic information allocation under uncertainty}}: a process in which an LLM balances procedural and epistemic information acquisition and then performs appropriate control actions based on that information.

\section{Experiments}

\subsection{Test-Time Control of Epistemic Tokens} \label{section:test_time_control}

Following the previous finding that injecting epistemic verbalization meaningfully recovers failed trajectories (Section~\ref{sec:injecting_epistemic}), we conversely manipulate the epistemic tokens defined in Section~\ref{sec:epistemic_tokens} at test time to further analyze their impact on reasoning performance. Specifically, we use DeepSeek-R1-Distill-Qwen-14B/32B~\citep{deepseek} and compare standard inference against inference in which epistemic tokens are suppressed.

\begin{figure}[ht!]
    \centering
    \includegraphics[width=\linewidth]{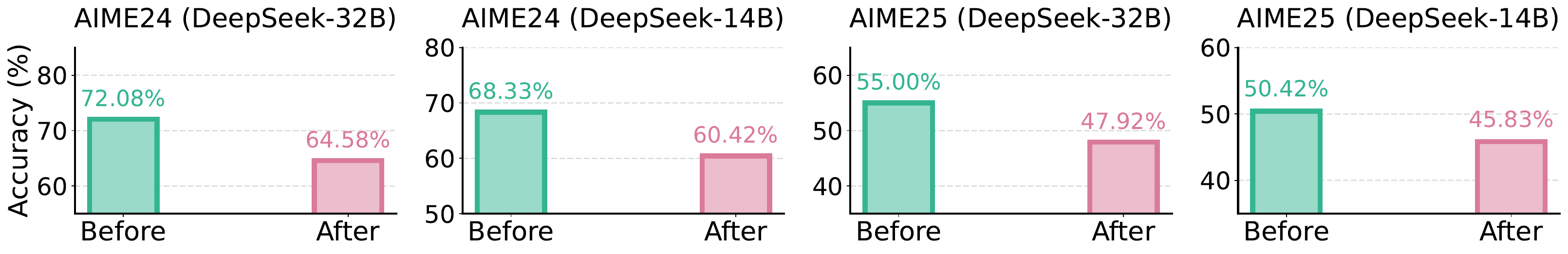}
    \caption{Comparison of avg@8 between baseline inference and inference with epistemic tokens suppressed on DeepSeek-R1-Distill-Qwen-14B/32B.}
    \label{fig:test_time}
    \vspace{-0.4cm}
\end{figure}

As shown in Figure~\ref{fig:test_time}, masking epistemic tokens  results in performance drops around 10\%. While suppression substantially disrupts uncertainty expression, performance does not collapse entirely. Analysis of the reasoning traces (Appendix~\ref{appendix:bypassing_suppression}) shows that the models sometimes route around the masked tokens by adopting alternative forms of epistemic verbalization.

\subsection{Suppressing Epistemic Verbalization via Distillation} \label{sec:ev_suppression}

To rule out such circumvention and isolate the effect of epistemic verbalization more cleanly, we design a controlled self-distillation SFT experiment. For each base model, we fine-tune it on 800 of its own reasoning traces generated with epistemic verbalization suppressed. Specifically, we append the instruction \emph{"Proceed directly without expressing uncertainty or doubt within the reasoning steps."} to the prompt, collect 800 problems the model answers correctly under this prompt, and use the resulting traces as SFT data. This yields a fine-tuned variant whose training traces differ from the base model's natural reasoning only in the absence of epistemic verbalization.

\newcolumntype{C}[1]{>{\centering\arraybackslash}m{#1}}
\begin{table}[h!]
    \centering
    \small
    \renewcommand{\arraystretch}{1.3}
    \begin{tabular}{p{3.5cm}|C{0.9cm}>{\columncolor[HTML]{FFF5F8}}C{2cm}}
    \hline
        & \textbf{Base} & \textbf{SFT (no EV)} \\
        \hline
        Qwen2.5-7B       & 13.3 & 6.7 \\
        Qwen3-8B-Base    & 16.7 & 3.3 \\
        Qwen3-14B-Base   & 16.7 & 10.0 \\
        DeepSeek-R1-Distill-7B      & 50.0  & 30.0 \\
        DeepSeek-R1-Distill-32B     & 80.0  & 43.3 \\
    \hline
    \end{tabular}
     \caption{AIME24 pass@1 across base models and SFT variants trained on traces with epistemic verbalization (EV) suppressed.}
    \label{table:ev_suppression}
\end{table}

As shown in Table~\ref{table:ev_suppression}, training on traces with epistemic verbalization suppressed consistently degrades performance, cutting accuracy by more than half in some cases despite the fact that all training traces contain correct answers. To better understand this degradation, we revisit the per-model capability map in Figure~\ref{fig:capability_map} that plots direct fluency against proactive correction rate. There, LRMs exhibit direct fluency comparable to that of LLMs but substantially higher proactive correction rates. Interestingly, suppressing epistemic verbalization through fine-tuning lowers their proactive correction rate, projecting the LRMs onto the LLM region of the map and bringing both their reasoning performance and behavior in line with those of LLMs. This further indicates that epistemic verbalization is not mere verbosity or stylistic decoration, but a substantive axis of reasoning capability.

\subsection{Impact of Epistemic Verbalization on Reasoning Distillation} \label{sec:limo_distillation}

We further examine the effects of distillation on a public dataset rich in epistemic verbalizations across several models, and find that epistemic verbalization is rapidly distilled and translates into quick performance gains when certain conditions are met, where the degree of distributional alignment on epistemic tokens plays a critical role. We use the LIMO-v2 dataset~\citep{limo}, which was gathered from DeepSeek R1, DeepSeek-R1-Distill-Qwen-32B \citep{deepseek}, and QwQ-32B \citep{qwen3}, and, as shown in Figure~\ref{fig:limo_epistemic}, contains a substantial number of epistemic tokens. Like our SFT setup in Section~\ref{sec:ev_suppression}, LIMO consists of only 800 examples, making it far too small to instill mathematical knowledge and instead serving primarily to reshape the model's linguistic habits.

\begin{figure}[h!]
    \centering
    \includegraphics[width=\linewidth]{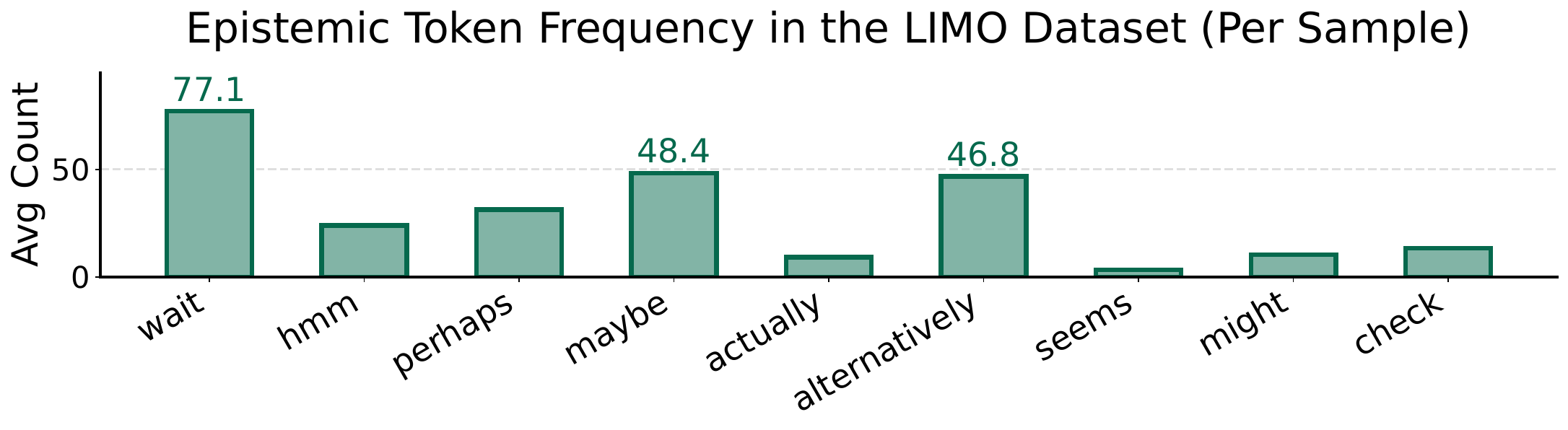}
    \caption{Per-sample epistemic token counts in LIMO. "Wait" alone appears 77 times per response on average.}
    \label{fig:limo_epistemic}
\end{figure}

Figure~\ref{fig:limo_aime} shows the AIME24 pass@1 scores of various models trained on the same LIMO dataset. For Qwen2.5-7B and Qwen3-8B/14B-Base, performance rapidly improves by up to 2.6x with only 800 samples, indicating that epistemic verbalization is not a special model capability but rather a linguistic habit that is easily distilled. In contrast, other models exhibit substantial performance degradation despite being trained on the same dataset with similar initial accuracy and model size. 

\begin{figure}[h!]
    \centering
    \includegraphics[width=0.32\linewidth]{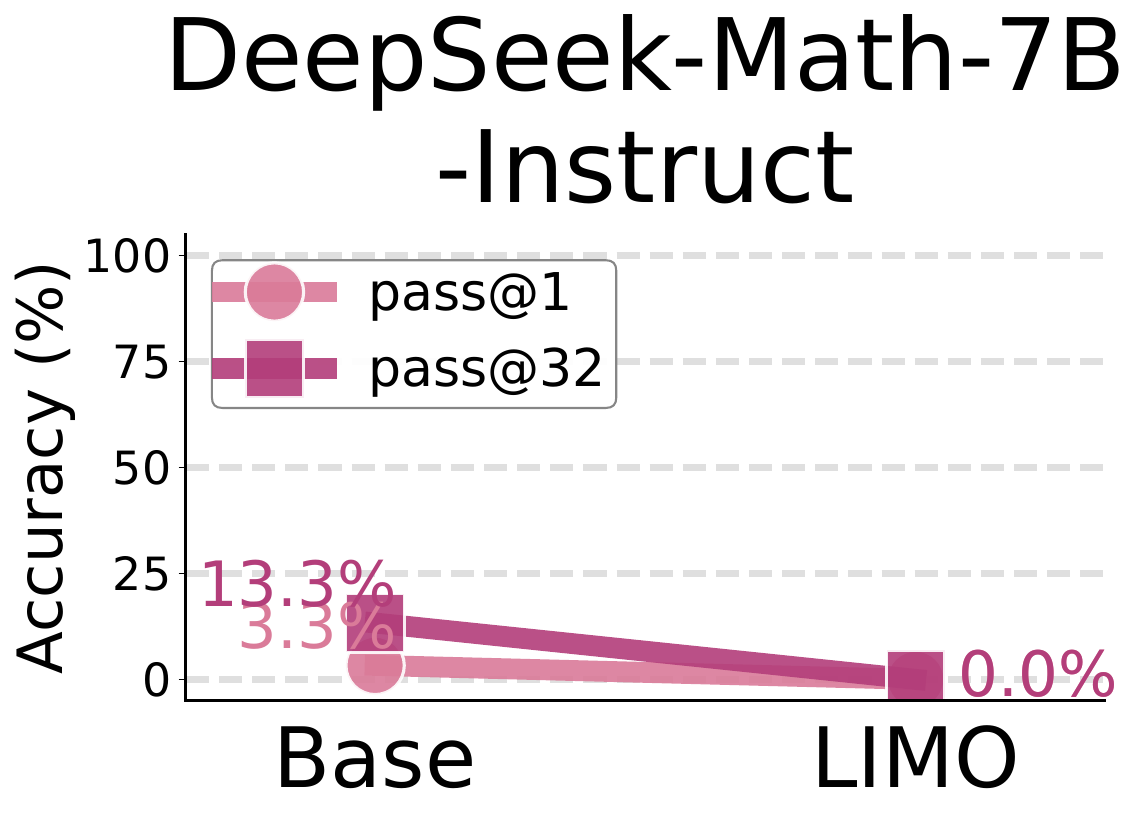}
    \includegraphics[width=0.32\linewidth]{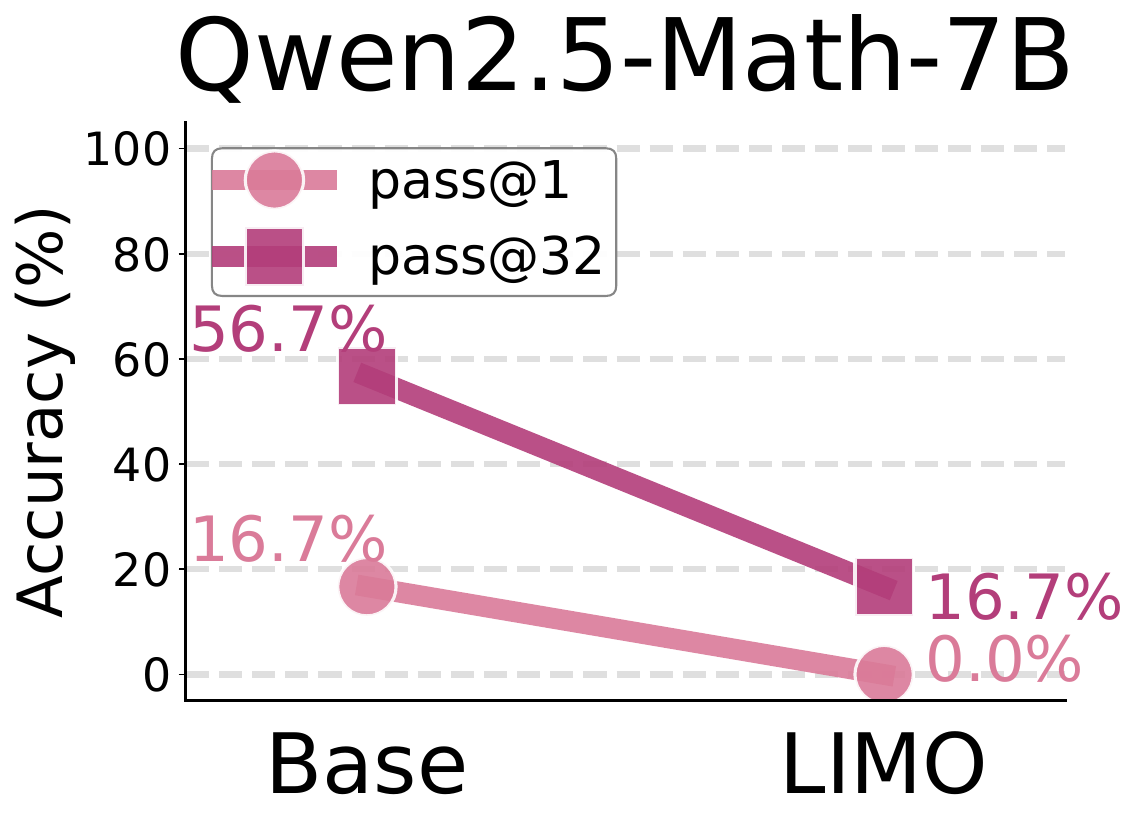}
    \includegraphics[width=0.32\linewidth]{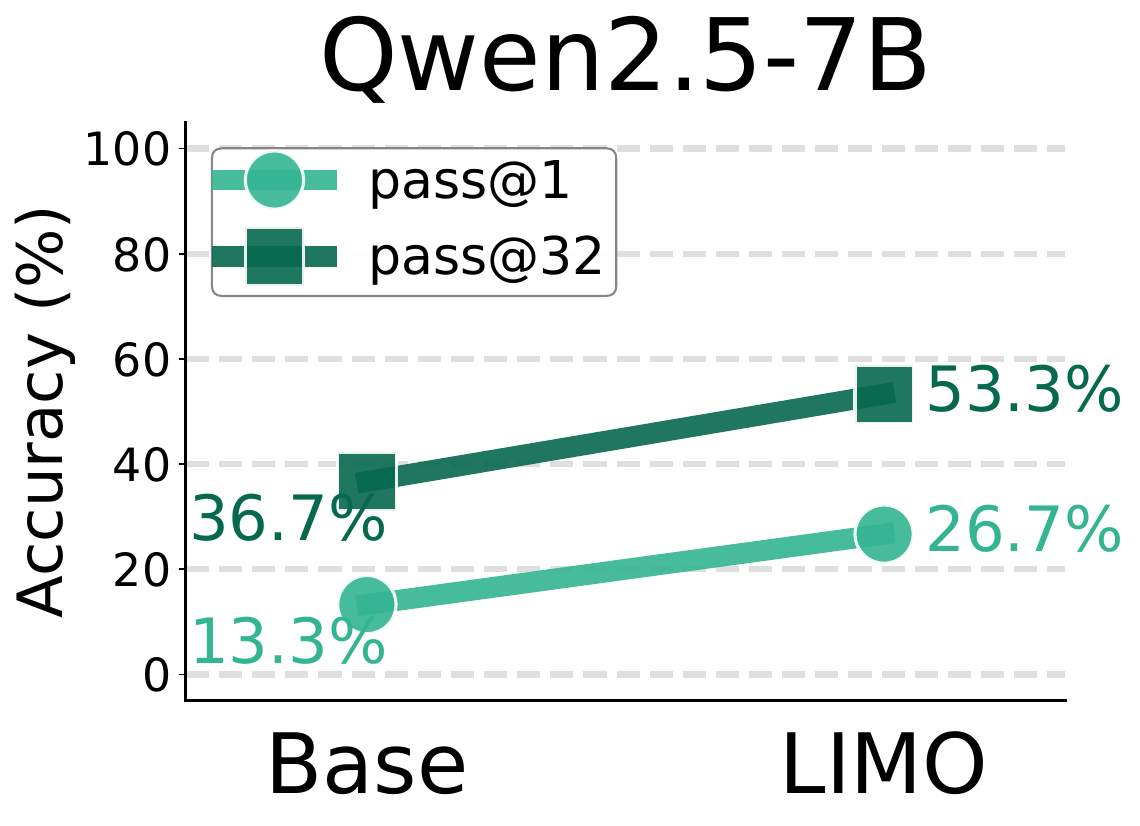} \\[10pt]
    \includegraphics[width=0.32\linewidth]{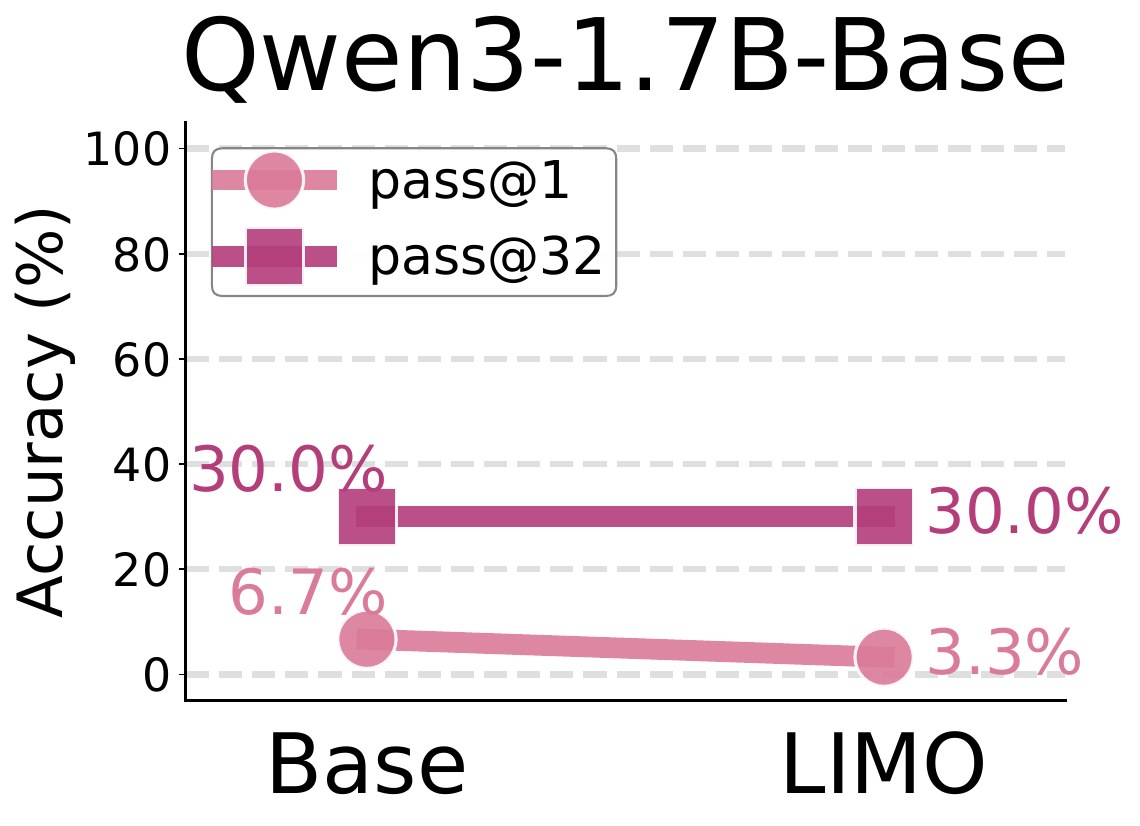}
    \includegraphics[width=0.32\linewidth]{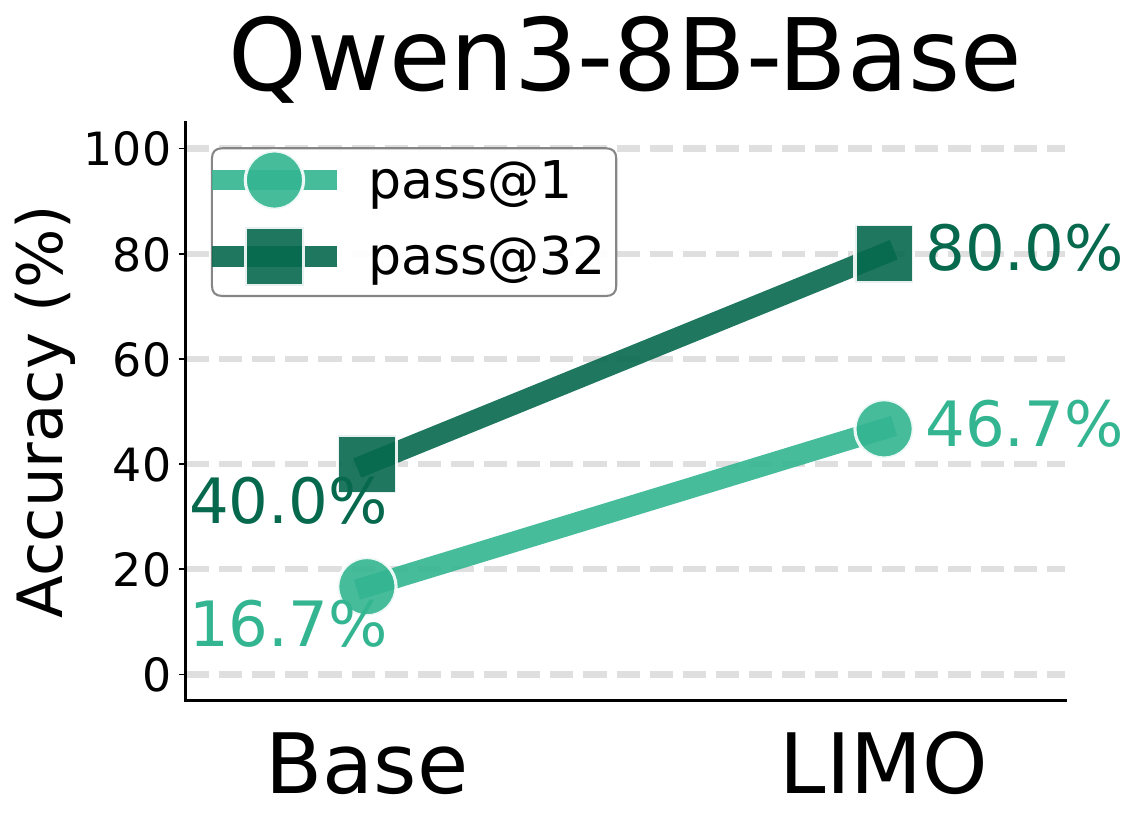}
    \includegraphics[width=0.32\linewidth]{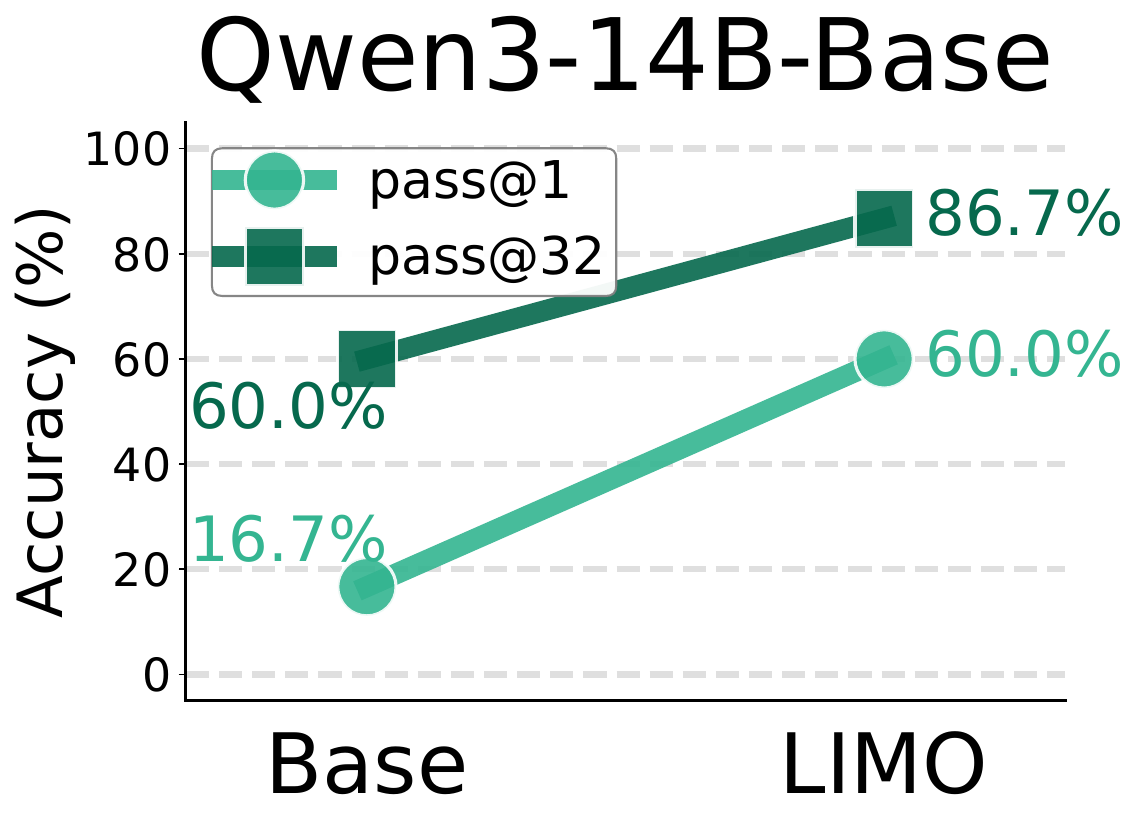} 
    \caption{Comparison of AIME24 pass@1 scores between base models and models trained with SFT on the LIMO-v2 dataset.}
    \label{fig:limo_aime}
\end{figure}

To explain this discrepancy, we analyze distributional alignment between the student models and the dataset via token-level log probabilities, focusing on how students evaluate frequent epistemic tokens such as "Wait" and "Alternatively." The contrast between successful and failed distillation is clear: as shown in Figure~\ref{fig:cumulative_distribution}, successful models exhibit well-aligned distributions in which epistemic tokens fall within the model's support, whereas poorly performing models show large gaps that place these tokens outside the support, hindering the adoption of the dataset's epistemic verbalization and control. Notably, in well-performing models, epistemic tokens remain low-probability and high-entropy relative to other tokens.

Together, the dispersion in distillation effectiveness from a high-reasoning model can be attributed to whether the base model's pre-existing characteristics are sufficiently "warmed up" to follow the high-reasoning model's epistemic verbalization. When the base model can readily absorb this ability, performance improves rapidly with only a small dataset, regardless of model size.

\begin{figure}[t!]
    \centering
    \includegraphics[width=\linewidth]{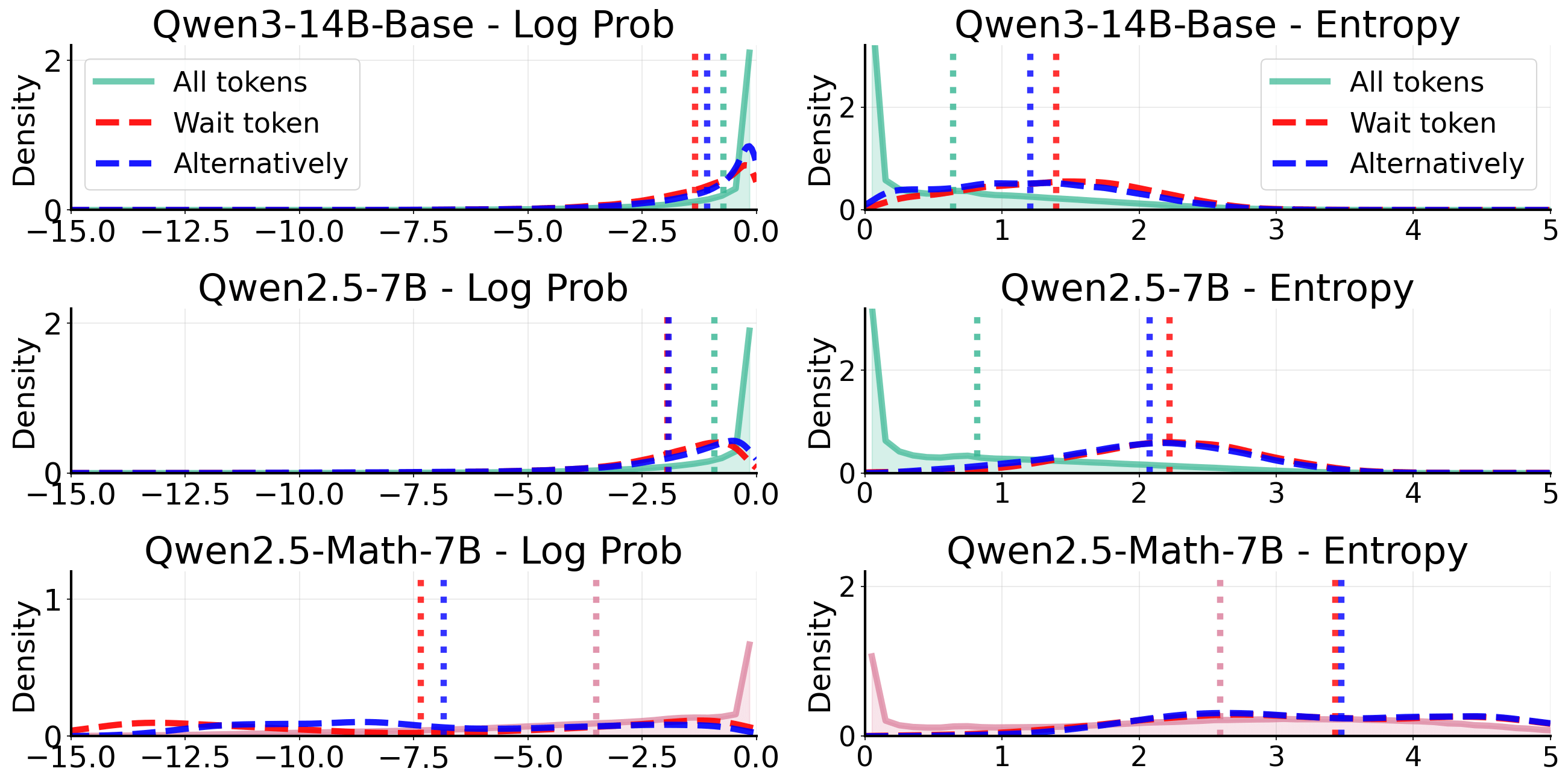}
    \caption{Distributions of token-level log-probability and entropy for all tokens versus the subset of "Wait" and "Alternative" tokens. Vertical blue and red lines indicate the dataset-level averages for "Wait" and "Alternative" tokens, respectively.}
    \label{fig:cumulative_distribution}

    \vspace{-0.3cm}
\end{figure}

\section{Conclusion}
In this work, we study effective reasoning in LLMs from an information-theoretic perspective, foregrounding the role of epistemic verbalization under uncertainty. We show that externalizing uncertainty enables continued information acquisition and self-correction, and reframe reasoning as a strategic allocation of information between 
procedural and epistemic axes. Empirically, we demonstrate that a minimal doubt cue suffices to recover failed trajectories, and that as few as 800 training samples can instill or suppress epistemic 
verbalization, suggesting it functions as a learnable linguistic habit rather than an intrinsic model capability. By distinguishing these two 
axes and their associated control actions, our framework clarifies otherwise elusive phenomena such as Aha moments and provides a general 
lens on contemporary LLM reasoning. We hope this perspective informs future theoretical work and inspires new approaches to modeling and controlling reasoning under uncertainty.

\section*{Limitations}

Our empirical analysis is mainly conducted on mathematical reasoning 
benchmarks, where correctness is objectively verifiable. Extending the 
analysis to other closed-world reasoning tasks and developing the 
formal world-Bayesian extension 
(Appendix~\ref{appendix:world_bayesian}) for open-world settings such 
as tool-augmented or interactive agents are left for future work. As discussed in Appendix \ref{appendix:world_bayesian}, external observations in such settings can independently surface latent errors, potentially reducing the model's reliance on epistemic verbalization. Empirically characterizing this trade-off is an interesting direction for future investigation.

Moreover, the 
nine epistemic tokens in Sections~\ref{sec:epistemic_tokens} 
and~\ref{section:test_time_control} serve as practical proxies and do not 
cover the full range of uncertainty expressions. While our core results 
do not rely on this specific set, a more thorough linguistic analysis 
of these proxies may be warranted. Finally, the classification of 
reactive and proactive correction, as well as the reasoning collapse 
analysis, depends on GPT-5 as an automated judge, introducing 
potential annotation noise.

\section*{Acknowledgments}
We sincerely thank Dohyung Kim and Sojeong Rhee for valuable discussions.
This work was supported by Microsoft Research.

\bibliography{custom}

@inproceedings{
limo,
title={{LIMO}: Less is More for Reasoning},
author={Yixin Ye and Zhen Huang and Yang Xiao and Ethan Chern and Shijie Xia and Pengfei Liu},
booktitle={Second Conference on Language Modeling},
year={2025},
url={https://openreview.net/forum?id=T2TZ0RY4Zk}
}

@article{limr,
  title={Limr: Less is more for rl scaling},
  author={Li, Xuefeng and Zou, Haoyang and Liu, Pengfei},
  journal={arXiv preprint arXiv:2502.11886},
  year={2025}
}

@inproceedings{
80_20,
title={Beyond the 80/20 Rule: High-Entropy Minority Tokens Drive Effective Reinforcement Learning for {LLM} Reasoning},
author={Shenzhi Wang and Le Yu and Chang Gao and Chujie Zheng and Shixuan Liu and Rui Lu and Kai Dang and Xiong-Hui Chen and Jianxin Yang and Zhenru Zhang and Yuqiong Liu and An Yang and Andrew Zhao and Yang Yue and Shiji Song and Bowen Yu and Gao Huang and Junyang Lin},
booktitle={The Thirty-ninth Annual Conference on Neural Information Processing Systems},
year={2025},
url={https://openreview.net/forum?id=yfcpdY4gMP}
}

@article{understanding_aha,
  title={Understanding aha moments: from external observations to internal mechanisms},
  author={Yang, Shu and Wu, Junchao and Chen, Xin and Xiao, Yunze and Yang, Xinyi and Wong, Derek F and Wang, Di},
  journal={arXiv preprint arXiv:2504.02956},
  year={2025}
}

@article{deepseek,
  title={Deepseek-r1: Incentivizing reasoning capability in llms via reinforcement learning},
  author={Guo, Daya and Yang, Dejian and Zhang, Haowei and Song, Junxiao and Zhang, Ruoyu and Xu, Runxin and Zhu, Qihao and Ma, Shirong and Wang, Peiyi and Bi, Xiao and others},
  journal={arXiv preprint arXiv:2501.12948},
  year={2025}
}

@article{qwen3,
  title={Qwen3 technical report},
  author={Yang, An and Li, Anfeng and Yang, Baosong and Zhang, Beichen and Hui, Binyuan and Zheng, Bo and Yu, Bowen and Gao, Chang and Huang, Chengen and Lv, Chenxu and others},
  journal={arXiv preprint arXiv:2505.09388},
  year={2025}
}

@article{qwen2.5,
  title={Qwen2. 5-math technical report: Toward mathematical expert model via self-improvement},
  author={Yang, An and Zhang, Beichen and Hui, Binyuan and Gao, Bofei and Yu, Bowen and Li, Chengpeng and Liu, Dayiheng and Tu, Jianhong and Zhou, Jingren and Lin, Junyang and others},
  journal={arXiv preprint arXiv:2409.12122},
  year={2024}
}

@article{step_dpo,
  title={Step-dpo: Step-wise preference optimization for long-chain reasoning of llms},
  author={Lai, Xin and Tian, Zhuotao and Chen, Yukang and Yang, Senqiao and Peng, Xiangru and Jia, Jiaya},
  journal={arXiv preprint arXiv:2406.18629},
  year={2024}
}

@inproceedings{step_by_step_math,
title={Step-by-Step Reasoning for Math Problems  via Twisted Sequential Monte Carlo},
author={Shengyu Feng and Xiang Kong and Shuang Ma and Aonan Zhang and Dong Yin and Chong Wang and Ruoming Pang and Yiming Yang},
booktitle={The Thirteenth International Conference on Learning Representations},
year={2025},
url={https://openreview.net/forum?id=Ze4aPP0tIn}
}

@article{chain_of_thought,
  title={Chain-of-thought prompting elicits reasoning in large language models},
  author={Wei, Jason and Wang, Xuezhi and Schuurmans, Dale and Bosma, Maarten and Xia, Fei and Chi, Ed and Le, Quoc V and Zhou, Denny and others},
  journal={Advances in neural information processing systems},
  volume={35},
  pages={24824--24837},
  year={2022}
}

@inproceedings{llamafactory,
  title={LlamaFactory: Unified Efficient Fine-Tuning of 100+ Language Models},
  author={Yaowei Zheng and Richong Zhang and Junhao Zhang and Yanhan Ye and Zheyan Luo and Zhangchi Feng and Yongqiang Ma},
  booktitle={Proceedings of the 62nd Annual Meeting of the Association for Computational Linguistics (Volume 3: System Demonstrations)},
  address={Bangkok, Thailand},
  publisher={Association for Computational Linguistics},
  year={2024},
  url={http://arxiv.org/abs/2403.13372}
}

@inproceedings{
understanding_cot,
title={Understanding Chain-of-Thought in {LLM}s through Information Theory},
author={Jean-Francois Ton and Muhammad Faaiz Taufiq and Yang Liu},
booktitle={Forty-second International Conference on Machine Learning},
year={2025},
url={https://openreview.net/forum?id=IjOWms0hrf}
}

@article{cot,
  title={Chain-of-thought prompting elicits reasoning in large language models},
  author={Wei, Jason and Wang, Xuezhi and Schuurmans, Dale and Bosma, Maarten and Xia, Fei and Chi, Ed and Le, Quoc V and Zhou, Denny and others},
  journal={Advances in neural information processing systems},
  volume={35},
  pages={24824--24837},
  year={2022}
}

@inproceedings{liang2025chain,
  title={Chain-of-thought reasoning for math: Theoretical foundation and applications},
  author={Liang, Jessica E},
  booktitle={2nd AI for Math Workshop@ ICML 2025},
  year={2025}
}

@article{yang2025decoupling,
  title={Decoupling Knowledge and Reasoning in LLMs: An Exploration Using Cognitive Dual-System Theory},
  author={Yang, Mutian and Gao, Jiandong and Wu, Ji},
  journal={arXiv preprint arXiv:2507.18178},
  year={2025}
}

@article{gan2025cot,
  title={CoT-Space: A Theoretical Framework for Internal Slow-Thinking via Reinforcement Learning},
  author={Gan, Zeyu and Yi, Hao and Liu, Yong},
  journal={arXiv preprint arXiv:2509.04027},
  year={2025}
}

@article{yong2025think,
  title={Think or Not? Exploring Thinking Efficiency in Large Reasoning Models via an Information-Theoretic Lens},
  author={Yong, Xixian and Zhou, Xiao and Zhang, Yingying and Li, Jinlin and Zheng, Yefeng and Wu, Xian},
  journal={arXiv preprint arXiv:2505.18237},
  year={2025}
}

@article{d2026illusion,
  title={The Illusion of Insight in Reasoning Models},
  author={d'Aliberti, Liv G and Ribeiro, Manoel Horta},
  journal={arXiv preprint arXiv:2601.00514},
  year={2026}
}

@inproceedings{
mi_peak,
title={Demystifying Reasoning Dynamics with Mutual Information: Thinking Tokens are Information Peaks in {LLM} Reasoning},
author={Chen Qian and Dongrui Liu and Haochen Wen and Zhen Bai and Yong Liu and Jing Shao},
booktitle={The Thirty-ninth Annual Conference on Neural Information Processing Systems},
year={2025},
url={https://openreview.net/forum?id=E1FrjgaG1J}
}

@article{tsui2025self,
  title={Self-Correction Bench: Revealing and Addressing the Self-Correction Blind Spot in LLMs},
  author={Tsui, Ken},
  journal={arXiv preprint arXiv:2507.02778},
  year={2025}
}

@inproceedings{
cannot_correct,
title={Large Language Models Cannot Self-Correct Reasoning Yet},
author={Jie Huang and Xinyun Chen and Swaroop Mishra and Huaixiu Steven Zheng and Adams Wei Yu and Xinying Song and Denny Zhou},
booktitle={The Twelfth International Conference on Learning Representations},
year={2024},
url={https://openreview.net/forum?id=IkmD3fKBPQ}
}

@inproceedings{tyen2024llms,
  title={LLMs cannot find reasoning errors, but can correct them given the error location},
  author={Tyen, Gladys and Mansoor, Hassan and C{\u{a}}rbune, Victor and Chen, Yuanzhu Peter and Mak, Tony},
  booktitle={Findings of the Association for Computational Linguistics: ACL 2024},
  pages={13894--13908},
  year={2024}
}

@inproceedings{
kamoi2024evaluating,
title={Evaluating {LLM}s at Detecting Errors in {LLM} Responses},
author={Ryo Kamoi and Sarkar Snigdha Sarathi Das and Renze Lou and Jihyun Janice Ahn and Yilun Zhao and Xiaoxin Lu and Nan Zhang and Yusen Zhang and Haoran Ranran Zhang and Sujeeth Reddy Vummanthala and Salika Dave and Shaobo Qin and Arman Cohan and Wenpeng Yin and Rui Zhang},
booktitle={First Conference on Language Modeling},
year={2024},
url={https://openreview.net/forum?id=dnwRScljXr}
}

@misc{maynot_aha,
  title={There May Not be Aha Moment in R1-Zero-like Training — A Pilot Study},
  author={Zichen Liu and Changyu Chen and Wenjun Li and Tianyu Pang and Chao Du and Min Lin},
  year={2025},
  howpublished={\url{https://oatllm.notion.site/oat-zero}},
  note={Notion Blog},
}

@article{yang2025dynamic,
  title={Dynamic early exit in reasoning models},
  author={Yang, Chenxu and Si, Qingyi and Duan, Yongjie and Zhu, Zheliang and Zhu, Chenyu and Li, Qiaowei and Chen, Minghui and Lin, Zheng and Wang, Weiping},
  journal={arXiv preprint arXiv:2504.15895},
  year={2025}
}

@article{hsic,
  title={A kernel statistical test of independence},
  author={Gretton, Arthur and Fukumizu, Kenji and Teo, Choon and Song, Le and Sch{\"o}lkopf, Bernhard and Smola, Alex},
  journal={Advances in neural information processing systems},
  volume={20},
  year={2007}
}

@inproceedings{reflact,
  title={Reflact: World-grounded decision making in llm agents via goal-state reflection},
  author={Kim, Jeonghye and Rhee, Sojeong and Kim, Minbeom and Kim, Dohyung and Lee, Sangmook and Sung, Youngchul and Jung, Kyomin},
  booktitle={Proceedings of the 2025 Conference on Empirical Methods in Natural Language Processing},
  pages={33421--33453},
  year={2025}
}

@inproceedings{
search-r1,
title={Search-R1: Training {LLM}s to Reason and Leverage Search Engines with Reinforcement Learning},
author={Bowen Jin and Hansi Zeng and Zhenrui Yue and Jinsung Yoon and Sercan O Arik and Dong Wang and Hamed Zamani and Jiawei Han},
booktitle={Second Conference on Language Modeling},
year={2025},
url={https://openreview.net/forum?id=Rwhi91ideu}
}

@article{agentmath,
  title={AgentMath: Empowering Mathematical Reasoning for Large Language Models via Tool-Augmented Agent},
  author={Luo, Haipeng and Feng, Huawen and Sun, Qingfeng and Xu, Can and Zheng, Kai and Wang, Yufei and Yang, Tao and Hu, Han and Tang, Yansong and Wang, Di},
  journal={arXiv preprint arXiv:2512.20745},
  year={2025}
}

@article{raise,
  title={Raise: Enhancing scientific reasoning in llms via step-by-step retrieval},
  author={Oh, Minhae and Kim, Jeonghye and Lee, Nakyung and Seo, Donggeon and Kim, Taeuk and Lee, Jungwoo},
  journal={arXiv preprint arXiv:2506.08625},
  year={2025}
}

@article{schick2023toolformer,
  title={Toolformer: Language models can teach themselves to use tools},
  author={Schick, Timo and Dwivedi-Yu, Jane and Dess{\`\i}, Roberto and Raileanu, Roberta and Lomeli, Maria and Hambro, Eric and Zettlemoyer, Luke and Cancedda, Nicola and Scialom, Thomas},
  journal={Advances in neural information processing systems},
  volume={36},
  pages={68539--68551},
  year={2023}
}

@article{yao2022react,
  title={React: Synergizing reasoning and acting in language models},
  author={Yao, Shunyu and Zhao, Jeffrey and Yu, Dian and Du, Nan and Shafran, Izhak and Narasimhan, Karthik and Cao, Yuan},
  journal={arXiv preprint arXiv:2210.03629},
  year={2022}
}

@article{song2025r1,
  title={R1-searcher: Incentivizing the search capability in llms via reinforcement learning},
  author={Song, Huatong and Jiang, Jinhao and Min, Yingqian and Chen, Jie and Chen, Zhipeng and Zhao, Wayne Xin and Fang, Lei and Wen, Ji-Rong},
  journal={arXiv preprint arXiv:2503.05592},
  year={2025}
}

\appendix

\section{Proof of Proposition~\ref{prop:epi_enables_sparse}}
\label{appendix:proof_convergence}

We first restate Proposition~\ref{prop:epi_enables_sparse} in a formal form that makes the dependence on $(\rho, \delta, \eta)$ explicit, and then provide a proof.

\begin{proposition}
\label{prop:epi_enables_sparse_formal}
Let $H_t := H(Y \mid \tilde S_t)$ and $\tau_\eta := \inf\{t \ge 0 : H_t \le \eta\}$. Suppose Assumption~\ref{assump:epi_info} holds, and that at every step $t$ with $H_{t-1} > \eta$, an epistemic token is generated with probability at least $\rho \in (0,1]$ (conditional on $\tilde S_{t-1}$). Then
\[
\mathbb{E}[\tau_\eta] \;\le\; \frac{\mathbb{E}[H_0]-\eta}{\rho\,\delta},
\]
\emph{regardless of the procedural trigger probability $p_E$}. If such pairs $(\rho(\eta), \delta(\eta))$ exist for every $\eta > 0$, then $\mathbb{E}[H_t] \to 0$ as $t \to \infty$.
\end{proposition}

\begin{proof}
Let $V_t \in \{0,1\}$ indicate whether an epistemic token is generated at step $t$, and let $\Delta_t := H_{t-1} - H_t \ge 0$. By Assumption~\ref{assump:epi_info} and the lower bound $\mathbb{P}(V_t = 1 \mid \tilde S_{t-1}) \ge \rho$, on the event $\{H_{t-1} > \eta\}$ we have
\[
\mathbb{E}[\Delta_t \mid \tilde S_{t-1}] \;\ge\; \delta \cdot \mathbb{P}(V_t = 1 \mid \tilde S_{t-1}) \;\ge\; \rho\,\delta,
\]
which depends only on $(\rho, \delta)$ and not on $p_E$. Telescoping up to $\tau_\eta$ and taking expectations,
\[
\mathbb{E}[H_0] - \eta \;\ge\; \mathbb{E}[H_0] - \mathbb{E}[H_{\tau_\eta}] \;\ge\; \rho\,\delta \cdot \mathbb{E}[\tau_\eta],
\]
which rearranges to the stated bound. If $(\rho(\eta), \delta(\eta))$ exist for every $\eta > 0$, then $\mathbb{E}[\tau_\eta] < \infty$ for all $\eta > 0$, so $\mathbb{E}[H_t] \to 0$.
\end{proof}

\section{World-Bayesian Reasoning with External Observations}
\label{appendix:world_bayesian}

We now extend the framework beyond the closed-world setting to reasoning processes that incorporate external observations. This includes embodied agents, tool-augmented language models, and interactive systems that acquire information from their environment during inference \citep{yao2022react, schick2023toolformer, search-r1, song2025r1, reflact, agentmath}.

In contrast to the closed-world, self-Bayesian setting, we consider \emph{world-Bayesian reasoning}, in which the model observes additional environmental signals that are statistically informative about the target variable. As we show below, the availability of external observations fundamentally changes the role of epistemic verbalization: it is no longer the sole mechanism for escaping silent divergence, which in turn clarifies why epistemic verbalization is so critical in the closed-world regime studied in the main text.

\paragraph{Setup.}
Let $x$ denote the initial input and $Y$ the target variable. At each reasoning step $t$, the agent may take an action $a_t$ and receive an external observation $o_t \in \mathcal{O}$, where
\[
  o_t \sim P(o_t \mid Y, a_t, s_{t-1}).
\]
The reasoning state now evolves as
\[
  s_t := (s_{t-1}, a_t, o_t), \qquad s_0 := x.
\]
Unlike internally generated tokens, observations $o_t$ constitute exogenous random variables that may introduce new information about $Y$. Each state $s_t$ induces a predictive distribution $P_\theta(Y \mid s_t)$, which may now be interpreted as an approximate Bayesian posterior updated through both self-conditioning and external evidence.

\paragraph{Information gain from external observation.}
The information gain associated with an external observation is given by
\begin{align*}
  \mathrm{IG}_{\mathrm{ext}}(s_t)
  &= H(Y \mid s_{t-1}) - H(Y \mid s_{t-1}, o_t) \\
  &= I(Y;\; o_t \mid s_{t-1}).
\end{align*}
This quantity is strictly non-negative and may be positive even when procedural self-conditioning yields no additional information. In particular, a single informative tool call or environmental observation can resolve an ambiguity that would otherwise require extended internal deliberation.

\paragraph{Reduced dependence on epistemic verbalization.}
In the closed-world setting, the central challenge is that silent divergence
produces no overt error, leaving epistemic verbalization as the \emph{only}
mechanism through which the model can regain traction
(Proposition~\ref{prop:epi_enables_sparse}). The world-Bayesian setting
relaxes this constraint. External observations can \emph{directly} surface
errors that remain latent under purely procedural reasoning. A tool call
that returns a contradictory result, or an environmental signal that
conflicts with the model's current trajectory, provides an exogenous
trigger for correction, one that does not depend on the model's ability
to verbalize its own uncertainty. In this regime, the burden of escaping
silent divergence is shared between the epistemic channel and the external
observation channel, reducing the model's reliance on the former.

\paragraph{Residual role of epistemic verbalization.}
Despite its diminished necessity, epistemic verbalization retains a residual role in the world-Bayesian setting: it serves as an internal monitoring signal that can trigger the \emph{decision} to seek external information. Recognizing that one's current trajectory is uncertain is a prerequisite for deciding to query a tool, request clarification, or perform an experiment. In this sense, epistemic verbalization shifts from being the primary mechanism of uncertainty reduction to a secondary role as a \emph{dispatch signal} that routes the policy toward external actions when internal reasoning is insufficient.

\paragraph{Relationship to the main framework.}
The closed-world framework arises as a special case when $\mathcal{O} = \emptyset$. In this regime, all uncertainty reduction must come from internal belief transformation, making epistemic verbalization the sole escape route from silent divergence. The contrast with the world-Bayesian setting highlights why epistemic verbalization plays such a disproportionate role in the closed-world regime: these LLMs lack external corrective signals entirely. Conversely, in tool-augmented or interactive systems, the performance gap between models with and without epistemic verbalization may narrow as the external observation channel becomes richer. Empirically characterizing this trade-off is an interesting direction for future work.

\section{Limits of Internal Uncertainty in Procedural Reasoning.}

A natural response to the limitations of procedural reasoning is to leverage model uncertainty as a corrective signal. Token-level entropy,
\begin{equation}
    H(A_t \mid s_{t-1}) = -\mathbb{E}_{a' \sim P_\theta(\cdot \mid s_{t-1})}\bigl[\log P_\theta(a' \mid s_{t-1})\bigr].
\end{equation}
has been widely studied as a proxy for reasoning uncertainty~\citep{yong2025think, yang2025dynamic, d2026illusion}. While token-level entropy can serve as a useful proxy for the model's internal uncertainty, it captures local confidence over the next token rather than uncertainty about the target variable $Y$, and this gap limits its effectiveness as a corrective signal.
\label{sec:internal_limit}
\begin{figure}[h!]
    \centering
    \includegraphics[width=0.493\linewidth]{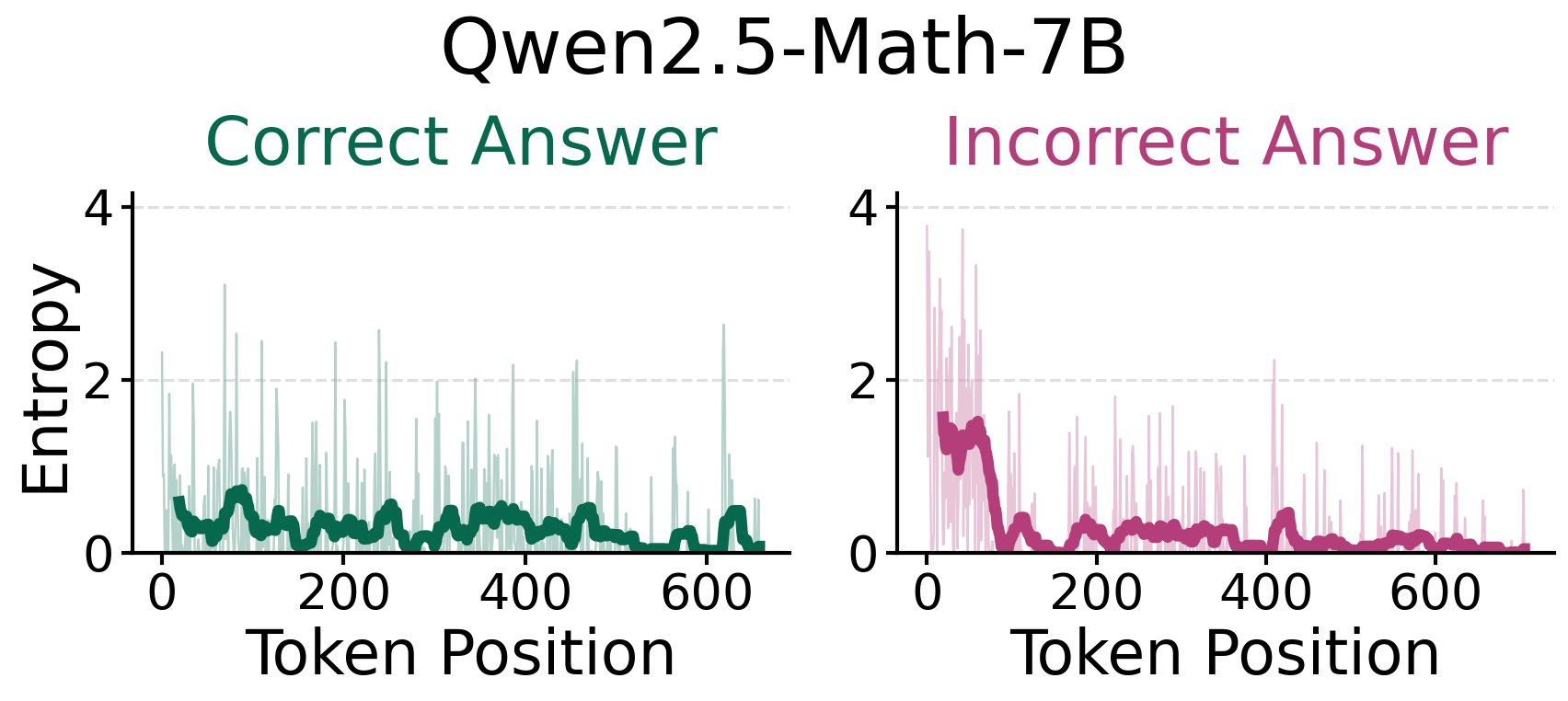}
    \includegraphics[width=0.493\linewidth]{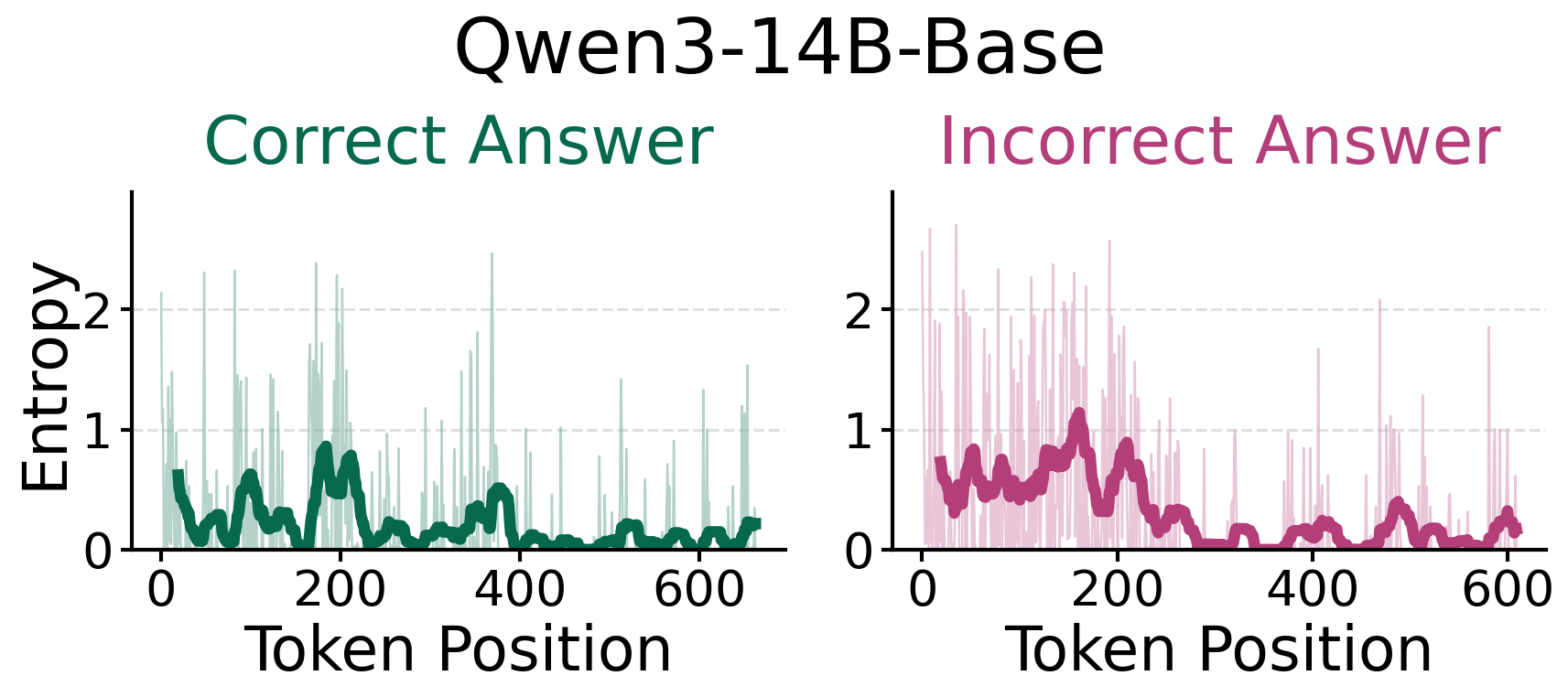}
    \caption{Token-level entropy over reasoning steps for Qwen2.5-Math-7B and Qwen3-14B-Base on AIME24 decreases similarly in both correct and incorrect solutions, suggesting that entropy alone does not reliably reflect uncertainty toward the correct answer.}
    \label{fig:token_level_uncertainty}
\end{figure}

Figure~\ref{fig:token_level_uncertainty} illustrates a case where token-level entropy fails to distinguish correct from incorrect reasoning: the model becomes locally more confident at similar rates regardless of whether its trajectory is productive, suggesting that entropy alone may not suffice as a corrective signal.

\section{Details of the Analysis}

\subsection{Reasoning Collapse Analysis}
\label{appendix:reasoning_collapse}

This appendix details the procedure behind the reasoning-collapse
statistics in Section~\ref{sec:procedural_limit}.

\paragraph{Pipeline.}
For each of the eight models analyzed in the main text, we sample
$N{=}8$ traces per problem across AIME24/25, AMC23, and MATH500
($4{,}800$ responses per model). Incorrect responses are passed to
an LLM judge (GPT-5-2025-08-07) using the prompt below, which returns a JSON object
indicating whether the trace exhibits collapse, its dominant type,
the onset sentence, and a brief justification. The collapse rate is
the fraction of incorrect responses flagged as collapsed.

\paragraph{Per-model statistics.}
Table~\ref{tab:collapse_stats} reports aggregate results on the Qwen
family. Collapse accounts for $54$--$62\%$ of all incorrect
responses ($58.4\%$ overall). Scaling lowers the total error rate
but leaves the share of collapse among remaining errors roughly
unchanged, indicating that larger models reduce surface-level
mistakes faster than silent divergence.

\begin{table}[h]
\centering
\small
\setlength{\tabcolsep}{4pt}
\begin{tabular}{p{2.4cm}|C{0.9cm}C{1.1cm}C{0.9cm}C{1cm}}
\toprule
Model & \#Inc. & Err.\ Rate & \#Col. & Col./Inc. \\
\midrule
Qwen2.5-7B       & 2{,}568 & 53.5\% & 1{,}452 & 56.5\% \\
Qwen2.5-32B      & 1{,}202 & 25.0\% &   717   & 59.7\% \\
Qwen3-1.7B-Base  & 2{,}486 & 51.8\% & 1{,}534 & 61.7\% \\
Qwen3-8B-Base    & 1{,}987 & 41.4\% & 1{,}091 & 54.9\% \\
Qwen3-14B-Base   & 1{,}401 & 29.2\% &   839   & 59.9\% \\
\midrule
Total            & 9{,}644 & 40.2\% & 5{,}633 & 58.4\% \\
\bottomrule
\end{tabular}
\caption{Reasoning collapse statistics. Each model is evaluated on
$600$ problems with $8$ sampled responses. "Inc." = incorrect
responses, "Col." = collapsed responses.}
\label{tab:collapse_stats}
\end{table}

\paragraph{Judge prompt.} We used GPT-5 with the prompt below to detect reasoning collapse.

\begin{tcolorbox}[
    colback=white,
    title=\small{Prompt for Reasoning Collapse Detection},
    colframe=gray,
    coltitle=black,
    colbacktitle=gray!10,
    breakable,
]
\small
\textbf{[System Prompt]}
\medskip

You are an expert at analyzing mathematical reasoning traces.
Your task is to detect reasoning collapse patterns in failed solution
attempts.
\tcblower
\small
\textbf{[User Prompt]}
\medskip

Analyze the following INCORRECT math solution and determine if it
exhibits a reasoning collapse pattern. Reasoning collapse includes
ANY of the following:
\begin{enumerate}[itemsep=0pt, topsep=0pt, parsep=0pt, partopsep=0pt, leftmargin=*]
    \item \textbf{REPETITION}: Repeating the same sentences, phrases,
    or steps in a loop.
    \item \textbf{TOPIC DRIFT}: Suddenly solving a completely
    different or made-up problem.
    \item \textbf{INCOHERENCE}: Statements become logically
    disconnected or nonsensical mid-solution.
    \item \textbf{INFINITE LOOP}: Getting stuck cycling through the
    same dead-end approach repeatedly without progress.
    \item \textbf{HALLUCINATION SPIRAL}: Inventing fictitious
    theorems, formulas, or numeric facts and building on them.
\end{enumerate}
Respond with a JSON object only:

\begin{lstlisting}[
    basicstyle=\small\ttfamily,
    breaklines=true,
    breakatwhitespace=false,
    columns=fullflexible,
]
{
  "has_collapse": true/false,
  "confidence": "high"/"medium"/"low",
  "collapse_type": "repetition"/"topic_drift"/"incoherence"
                   /"infinite_loop"/"hallucination_spiral"/null,
  "collapse_onset": "the sentence or phrase where collapse begins",
  "reasoning": "1-2 sentence explanation of your judgment"
}
\end{lstlisting}
\medskip
\textbf{Question:} \{question\}
\medskip

\textbf{Solution:} \{solution\}
\medskip

\textbf{Correct Answer:} \{gold\_answer\}
\end{tcolorbox}

\subsection{Reactive and Proactive Correction Analysis}
\label{appendix:reactive_proactive_correction}

Next, we provide details of the reactive and proactive correction classification in Section \ref{sec:collapse_escape}.

\paragraph{Stage~1: Self-correction detection.} To categorize all self-correction cases observed per model, we first apply the prompt to all correct-answer responses.

\begin{tcolorbox}[
    colback=white,
    title=\small{Prompt for Self-Correction Pattern Detection},
    colframe=gray,
    coltitle=black,
    colbacktitle=gray!10,
    breakable,
]
\small
\textbf{[System Prompt]}
\medskip
You are an expert at analyzing mathematical reasoning traces.
Your task is to determine whether a solution demonstrates a specific
self-correction pattern.
\tcblower
\small
\textbf{[User Prompt]}
\medskip
Analyze the following math solution and determine if it exhibits the
self-correction pattern.

The self-correction pattern requires ALL of the following:
\begin{enumerate}[itemsep=0pt, topsep=0pt, parsep=0pt, partopsep=0pt, leftmargin=*]
    \item The solver initially pursues a \textbf{WRONG} approach,
    method, or makes a significant mistake.
    \item The solver \textbf{RECOGNIZES} the error.
    \item The solver \textbf{ABANDONS} the wrong path and takes a
    different approach.
    \item The solver arrives at the \textbf{CORRECT} final answer.
\end{enumerate}
Respond with a JSON object only:
\begin{lstlisting}[
    basicstyle=\small\ttfamily,
    breaklines=true,
    breakatwhitespace=false,
    columns=fullflexible,
]
{
  "has_self_correction": true/false,
  "confidence": "high"/"medium"/"low",
  "wrong_path_description": "brief description of the initial wrong approach",
  "correction_trigger": "the exact phrase/sentence where the solver recognizes the error",
  "reasoning": "1-2 sentence explanation of your judgment"
}
\end{lstlisting}
\medskip
\textbf{Question:} \{question\}
\medskip

\textbf{Solution:} \{solution\}
\medskip

\textbf{Correct Answer:} \{gold\_answer\}
\end{tcolorbox}

\paragraph{Stage~2: Trigger classification (reactive vs.\ proactive).}

We then pass all self-correction cases identified above to the prompt below to
classify each correction as \emph{reactive} (evidence-based) or
\emph{proactive} (suspicion-based).

\begin{tcolorbox}[
    colback=white,
    title=\small{Prompt for Self-Correction Trigger Classification},
    colframe=gray,
    coltitle=black,
    colbacktitle=gray!10,
    breakable,
]
\small
\textbf{[System Prompt]}
\medskip
You are an expert at analyzing mathematical reasoning traces.
\tcblower
\small
\textbf{[User Prompt]}
\medskip
A language model was solving a math problem and performed
self-correction. Classify what triggered the self-correction into
one of two types:

\smallskip
\textbf{evidence\_based}: The model encountered a clear, explicit
signal that something was wrong BEFORE deciding to correct. Examples:
\begin{itemize}[itemsep=0pt, topsep=0pt, parsep=0pt, partopsep=0pt, leftmargin=*]
    \item A mathematical result that is impossible
    (negative angle, $\sin > 1$, division by zero).
    \item A result that explicitly violates a stated condition
    (e.g., $x=6$ but problem says $x>7$).
    \item A logical contradiction derived from the work
    (e.g., $6=3$, empty solution set, negative count).
    \item An exhausted search that found nothing
    (tried all cases, none worked).
\end{itemize}

\smallskip
\textbf{suspicion\_based}: The model decided to re-examine WITHOUT a
clear signal of error. Examples:
\begin{itemize}[itemsep=0pt, topsep=0pt, parsep=0pt, partopsep=0pt, leftmargin=*]
    \item Vague unease (``let me reconsider'', ``I'm not sure'',
    ``let me double-check'').
    \item The model just tried a different approach without
    identifying what was wrong.
    \item Re-examination triggered by aesthetic/intuitive concern
    rather than a concrete contradiction.
\end{itemize}

\smallskip
Respond with a JSON object only:
\begin{lstlisting}[
    basicstyle=\small\ttfamily,
    breaklines=true,
    breakatwhitespace=false,
    columns=fullflexible,
]
{
  "trigger_type": "evidence_based" or "suspicion_based",
  "confidence": "high" or "medium" or "low",
  "evidence": "the specific contradiction/signal that triggered correction (if evidence_based), or null",
  "reasoning": "1 sentence explanation"
}
\end{lstlisting}
\end{tcolorbox}

\subsection{Analysis of the Precision of Proactive Signals}  

To analyze the precision of proactive signals in Section \ref{sec:proactive},
we additionally sample 80 responses per model (320 total across the four LRMs)
that contain at least one proactive-pattern verbalization, and verify via GPT-5
whether the model was actually on a wrong path at that point using the prompt below.

\begin{tcolorbox}[
    colback=white,
    title=\small{Prompt for Wrong-Path Verification at Doubt Point},
    colframe=gray,
    coltitle=black,
    colbacktitle=gray!10,
    breakable,
]
\small
\textbf{[System Prompt]}
\medskip
You are an expert at analyzing LLM mathematical reasoning traces.
\tcblower
\small
\textbf{[User Prompt]}
\medskip
A language model was solving a math problem. At the marked point
\texttt{[DOUBT]}, it expressed uncertainty or suspicion without an
explicit error signal.

\smallskip
Was the model actually on a \textbf{WRONG} reasoning path just before
\texttt{[DOUBT]}? Or was it on a \textbf{CORRECT} path expressing
unnecessary doubt?

\smallskip
\textbf{Question:} \{question\}

\textbf{Gold answer:} \{gold\_answer\}

\smallskip
\textbf{Reasoning excerpt (ends at \texttt{[DOUBT]}):}
\begin{lstlisting}[
    basicstyle=\small\ttfamily,
    breaklines=true,
    breakatwhitespace=false,
    columns=fullflexible,
]
"""
{excerpt}
[DOUBT: "{trigger}"]
"""
\end{lstlisting}

\smallskip
Respond with JSON only:
\begin{lstlisting}[
    basicstyle=\small\ttfamily,
    breaklines=true,
    breakatwhitespace=false,
    columns=fullflexible,
]
{
  "was_on_wrong_path": true | false,
  "confidence": "high" | "medium" | "low",
  "reasoning": "<one sentence>"
}
\end{lstlisting}
\end{tcolorbox}

\subsection{Identifying Epistemic Tokens}

To identify epistemic tokens in Section \ref{sec:epistemic_tokens}, we apply the characterization prompt below to all proactive (suspicion-based)
self-correction cases identified in Appendix \ref{appendix:reactive_proactive_correction}
for the four LRMs that exhibit meaningful proactive correction rates. GPT-5 is asked
to (i)~distill the minimal trigger tokens, (ii)~assign a free-form verbalization label,
and (iii)~describe the epistemic function expressed.

\label{appendix:identifying_epistemic_tokens}

\begin{tcolorbox}[
    colback=white,
    title=\small{Prompt for Suspicion Verbalization Characterization},
    colframe=gray,
    coltitle=black,
    colbacktitle=gray!10,
    breakable,
]
\small
\textbf{[System Prompt]}
\medskip
You are an expert at analyzing LLM reasoning traces.
Your task is to characterize the type of verbalization a language
model uses when it expresses uncertainty or doubt without any clear
evidence of error --- no mathematical contradiction, no violated
constraint, just vague unease.
\tcblower
\small
\textbf{[User Prompt]}
\medskip
A language model was solving a math problem. At some point it
expressed uncertainty or suspicion --- without any concrete evidence
that something was wrong --- using the following phrase.

\smallskip
Your tasks:
\begin{enumerate}[itemsep=0pt, topsep=0pt, parsep=0pt, partopsep=0pt, leftmargin=*]
    \item Extract the minimal ``trigger tokens'' --- the 1--5 word
    core phrase that most directly expresses the uncertainty
    (e.g.\ ``wait, maybe'', ``hmm'', ``i'm not sure'').
    \item Give this verbalization a short, descriptive label of your
    own choosing (2--4 words, snake\_case). Do NOT use any predefined
    list --- describe what you actually observe.
    \item In one sentence, describe what internal state the model is
    expressing through this verbalization.
\end{enumerate}

\smallskip
\textbf{Question context:} \{question\}

\textbf{Prior reasoning path:} \{wrong\_path\}

\textbf{Verbalization phrase:} ``\{trigger\}''

\textbf{Surrounding context:} \{reasoning\}

\smallskip
Respond with JSON only:
\begin{lstlisting}[
    basicstyle=\small\ttfamily,
    breaklines=true,
    breakatwhitespace=false,
    columns=fullflexible,
]
{
  "trigger_tokens": "<1-5 word core phrase, lowercased>",
  "verbalization_label": "<your 2-4 word descriptive label in snake_case>",
  "epistemic_function": "<one sentence describing what the model is expressing>",
  "confidence": "high" | "medium" | "low"
}
\end{lstlisting}
\end{tcolorbox}

\section{Details of the Experimental Setup}

\paragraph{Training hyperparameters.}

For all distillation experiments, all models are fine-tuned using LLaMA-Factory \cite{llamafactory} under the default LIMO configuration \cite{limo}. We utilize four B200 GPUs for training.

\begin{table}[h!]
    \centering
    \small
    \renewcommand{\arraystretch}{1.3}
    \caption{AIME24 pass@1 performance across various base models with varying temperature and top-p configurations.}
    \begin{tabular}{l|p{0.9cm}p{0.85cm}p{0.85cm}}
        \hline
        Temperature (Top-P) & 0.0 & 0.7 (1.0) & 0.7 (0.8) \\
        \hline
        DeepSeek-R1-Distill-32B & 50.0 & 70.0 & \textbf{80.0} \\
        DeepSeek-R1-Distill-14B & 56.67 & 56.67 & \textbf{70.0} \\
        DeepSeek-R1-Distill-1.5B & 20.0 & \textbf{30.0} & \textbf{30.0} \\
        \hline
        DeepSeek-Math-7B\scriptsize{-Instruct} & 3.33 & 0.0 & 0.0\\
        \hline
        Qwen2.5-Math-7B & \textbf{16.67} & 10.0 & 6.67\\
        Qwen2.5-7B & \textbf{13.33} &  6.67 & 3.33 \\
        \hline
        Qwen3-1.7B-Base  & 0.0 & 0.0 & 6.67 \\
        Qwen3-4B-Base & 6.67 &  \textbf{13.33} & 6.67 \\
        Qwen3-8B-Base & 13.33 & \textbf{16.67} & 10.0 \\
        Qwen3-14B-Base  & \textbf{16.67} & 10.0 & 13.33\\
        \hline
         
    \end{tabular}
    \label{table:temperature_ablation}
\end{table}

\paragraph{Selection of decoding hyperparameters.}
We choose the decoding temperature mainly as 0.0 for calculating pass@1 and 0.7 for calculating pass@k or acc@16. However, there are some concerns, especially for recent models that exhibit more explicit reasoning behavior, such as DeepSeek or Qwen3, that a temperature of 0.0 may lead to performance degradation. Since we could not find detailed guidance for the Qwen3-base model, we additionally conduct an ablation study with varying temperatures to verify this effect when computing the pass@1 score.

As shown in Table \ref{table:temperature_ablation}, a temperature of 0.0 is harmful for the DeepSeek-Distill models, while it is beneficial for the Qwen2.5 models. For the Qwen3-Base models, the effect varies across settings. In particular, for the DeepSeek-Distill models, setting Top-P to 0.8 provides a substantial performance gain. Following the results in Table \ref{table:temperature_ablation}, we report the pass@1 performance of the base model using the larger value between temperatures 0.0 and 0.7, while the LIMO pass@1 performance is always reported with the temperature fixed at 0.0.

\section{Epistemic Verbalization Produces Information Gain}
\label{appendix:epistemic_verbalization_informativeness}

\citet{mi_peak} observe that, in LRM reasoning, most steps carry little mutual information (MI) with the correct answer, while a small number of steps exhibit sharp MI increases (``MI peaks'') that frequently coincide with thinking tokens such as ``Wait''. In this section, we extend their analysis to examine whether \textbf{epistemic verbalization indeed induces information gain}.

\paragraph{Setup.}
Building on the official implementation of \citet{mi_peak} (\url{https://github.com/ChnQ/MI-Peaks}), we measure the statistical dependence between the hidden representation at each token position and the ground-truth answer representation. Specifically, for each token position $t$ in the model-generated reasoning trajectory, we extract the last-layer hidden state $\mathbf{h}_t$ and compute its mutual information with the final-token representation of the ground-truth answer, $\mathbf{h}_{\text{GT}}$, using an HSIC-based estimator~\citep{hsic}. This enables \textbf{token-level tracking of when and how answer-relevant information emerges} in the model's internal representations during reasoning.

\paragraph{Results.}
To examine whether epistemic verbalization drives information gain, we compare two models from the same family (Qwen3-8B-Base) on the same problem: one that recovers the correct answer through self-correction, and one that fails to recover. Since epistemic verbalization and self-correction are rarely observed in Qwen3-Base, we additionally include \textbf{Qwen3-8B-SFT}, fine-tuned from the same base on reasoning traces rich in epistemic verbalization~\citep{limo} (as shown in Section~\ref{sec:limo_distillation}), to elicit such behavior at inference time. On AIME24 \#7, both models initially commit to an incorrect trajectory, but only \textbf{Qwen3-8B-SFT recovers via self-correction}, while Qwen3-Base remains on the incorrect path (trajectories in Appendix~\ref{appendix:base_limo_trajectory}).

\begin{figure}[h!]
    \centering
    \includegraphics[width=0.95\linewidth]{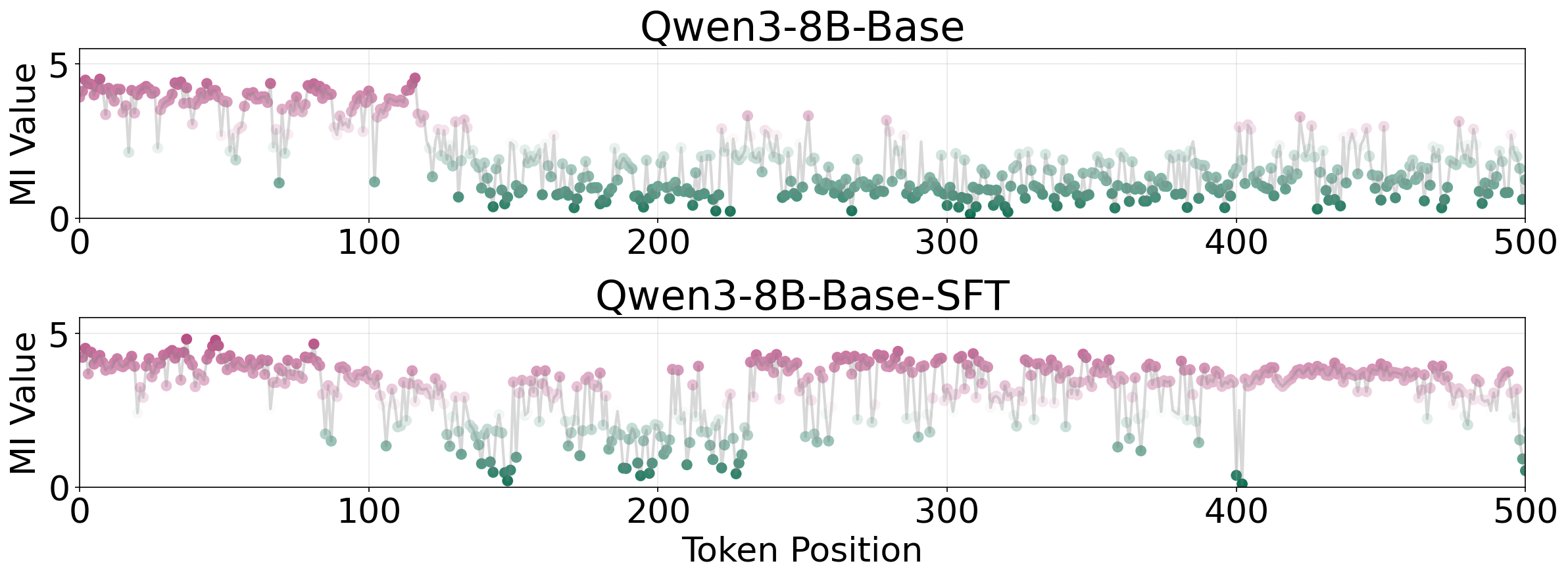}
    \caption{On AIME24 \#7, both models initially fail, but only 
    Qwen3-8B-SFT sustains information gain through epistemic verbalization 
    and recovers the correct answer.}
    \label{fig:mi_trajectory}
\end{figure}

As shown in Figure~\ref{fig:mi_trajectory}, Qwen3-8B-Base's MI collapses to near zero shortly after divergence, whereas Qwen3-8B-SFT sustains elevated MI while producing evaluative expressions such as ``Wait, let me check.'' A closer inspection of high-MI regions reveals that MI tends to rise at thinking tokens, but not always. Instead, elevated MI appears in utterances that perform epistemic verbalization, and thinking tokens occurring \emph{without} such verbalization (e.g., "Hmm" in the left panel of Figure~\ref{fig:mi_peak}) do not trigger an MI increase.

\begin{figure}[h!]
    \centering
    \includegraphics[width=\linewidth]{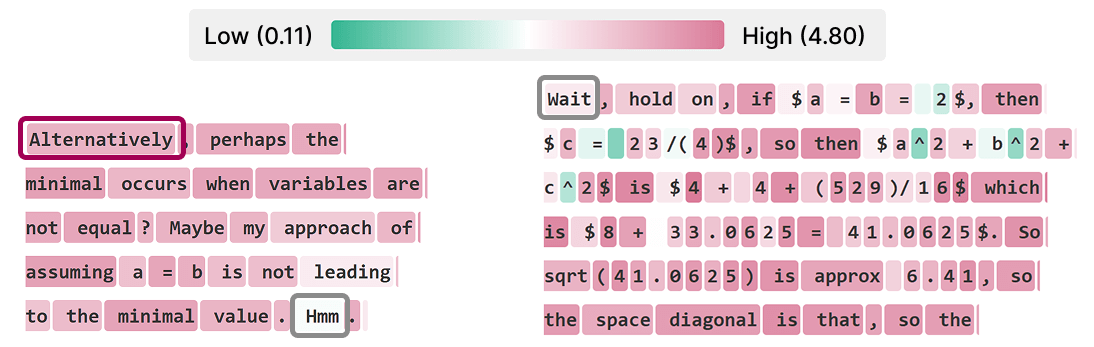}
    \caption{Token-level MI analysis: high MI aligns with evaluative 
    behavior, not with the thinking tokens themselves.}
    \label{fig:mi_peak}
\end{figure}

Aligning with our framework and the preceding discussion, these tokens are surface manifestations of a deeper mechanism. They externalize uncertainty, rendering the model's epistemic state conditionable and reusable during inference.


\section{More Relationship Between Uncertainty and Epistemic Verbalization} \label{appendix:more_epistemic}

We further investigate whether the uncertainty expressed during reasoning truly reflects the model's underlying uncertainty. Leveraging the observation that more challenging problems tend to elicit greater uncertainty, we analyze how uncertainty is verbalized in the outputs of strong reasoning models that exhibit Aha-moment or self-reflective behaviors.

\paragraph{DeepSeek-R1-Distill-Qwen.} We first study DeepSeek-R1-Distill-Qwen models ranging from 1.5B to 14B parameters, comparing their performance and response length.

\begin{figure}[ht!]
    \centering
    \includegraphics[width=\linewidth]{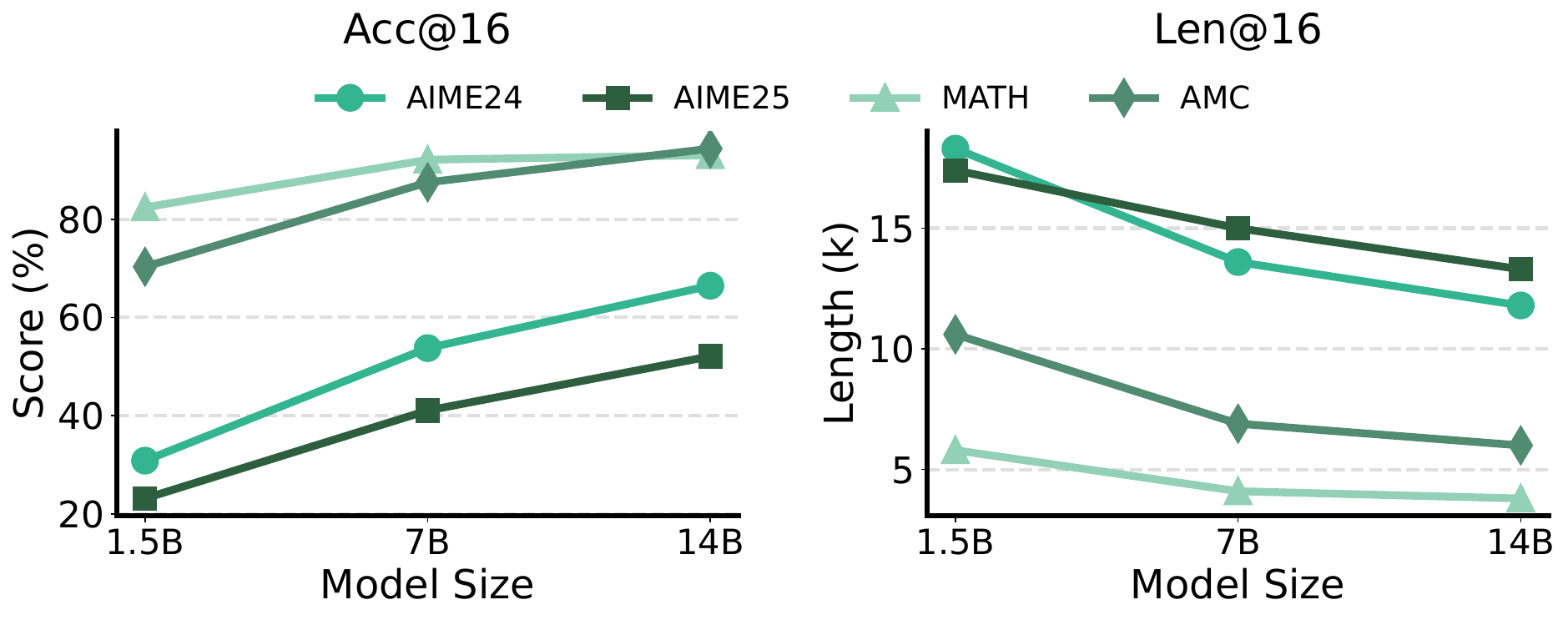}
    \caption{Acc@16 (avg. score) and Len@16 (avg. response length) of DeepSeek-Distill 1.5B–14B on math benchmarks.}
    \label{figure:uncertainty_epistemic}
\end{figure}

As shown in Figure~\ref{figure:uncertainty_epistemic}, more difficult problems (AIME24/25) elicit longer responses than easier ones (AMC23/MATH500). Increasing model size is associated with higher scores and shorter responses.

We further examined how frequently DeepSeek-R1-Distill-Qwen-\{1.5B, 7B, 14B\} generate epistemic tokens during solution generation. As shown in Figure~\ref{figure:uncertainty_epistemic}, smaller models use these tokens more frequently: relative to the 14B model, the 1.5B model produces 75\% more occurrences of "Wait" and 235\% more of "Perhaps" in AIME24. This suggests that when reasoning models face problems beyond their capacity, they verbalize epistemic uncertainty more often. We also observe that epistemic tokens appear more frequently on harder benchmarks (AIME24/25) and less frequently on easier ones (AMC23/MATH500), indicating that these tokens indeed reflect the uncertainty the model experiences during reasoning.

\begin{figure}[h!]
    \centering
    \includegraphics[width=0.49\linewidth]{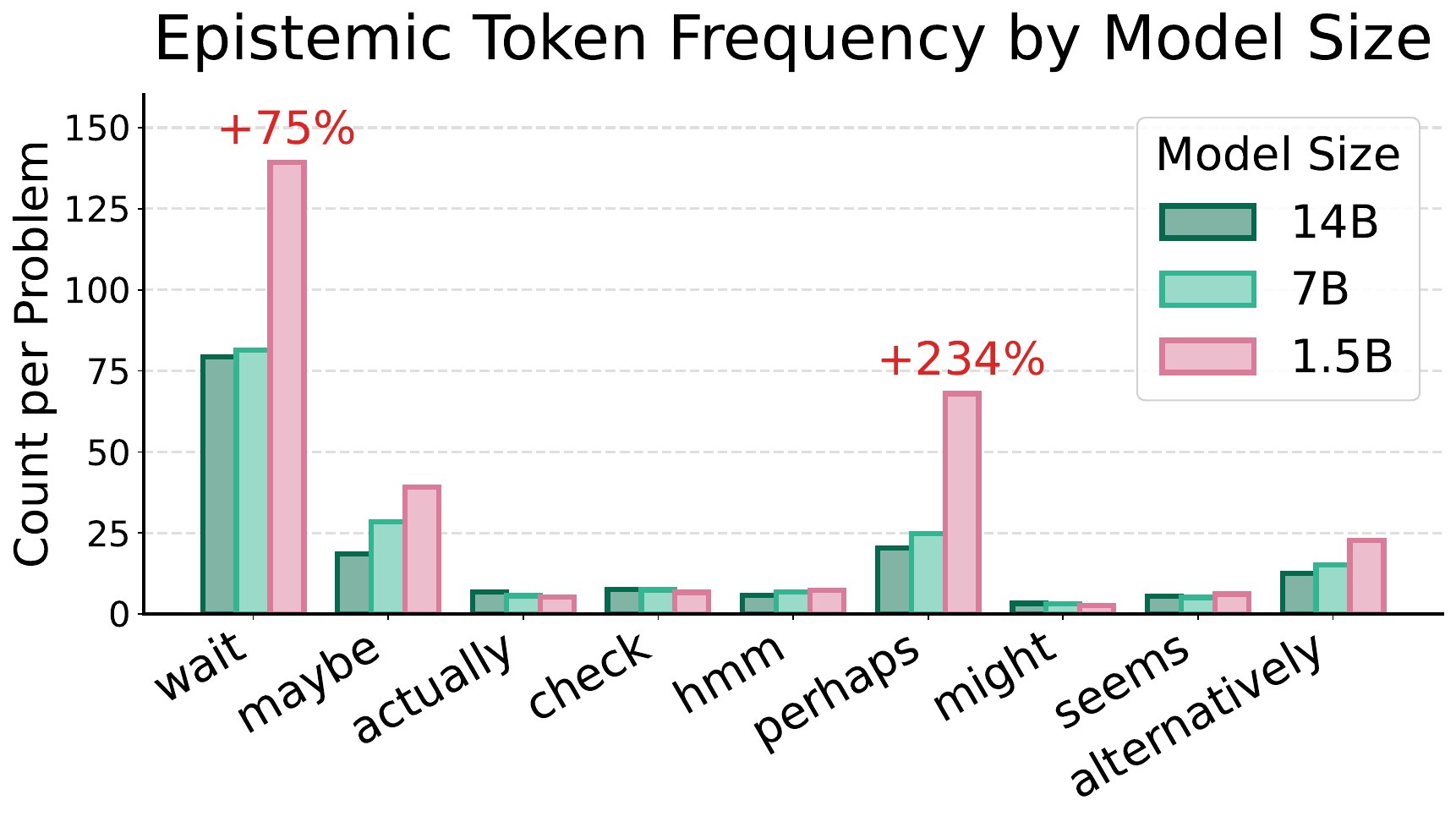}
    \includegraphics[width=0.49\linewidth]{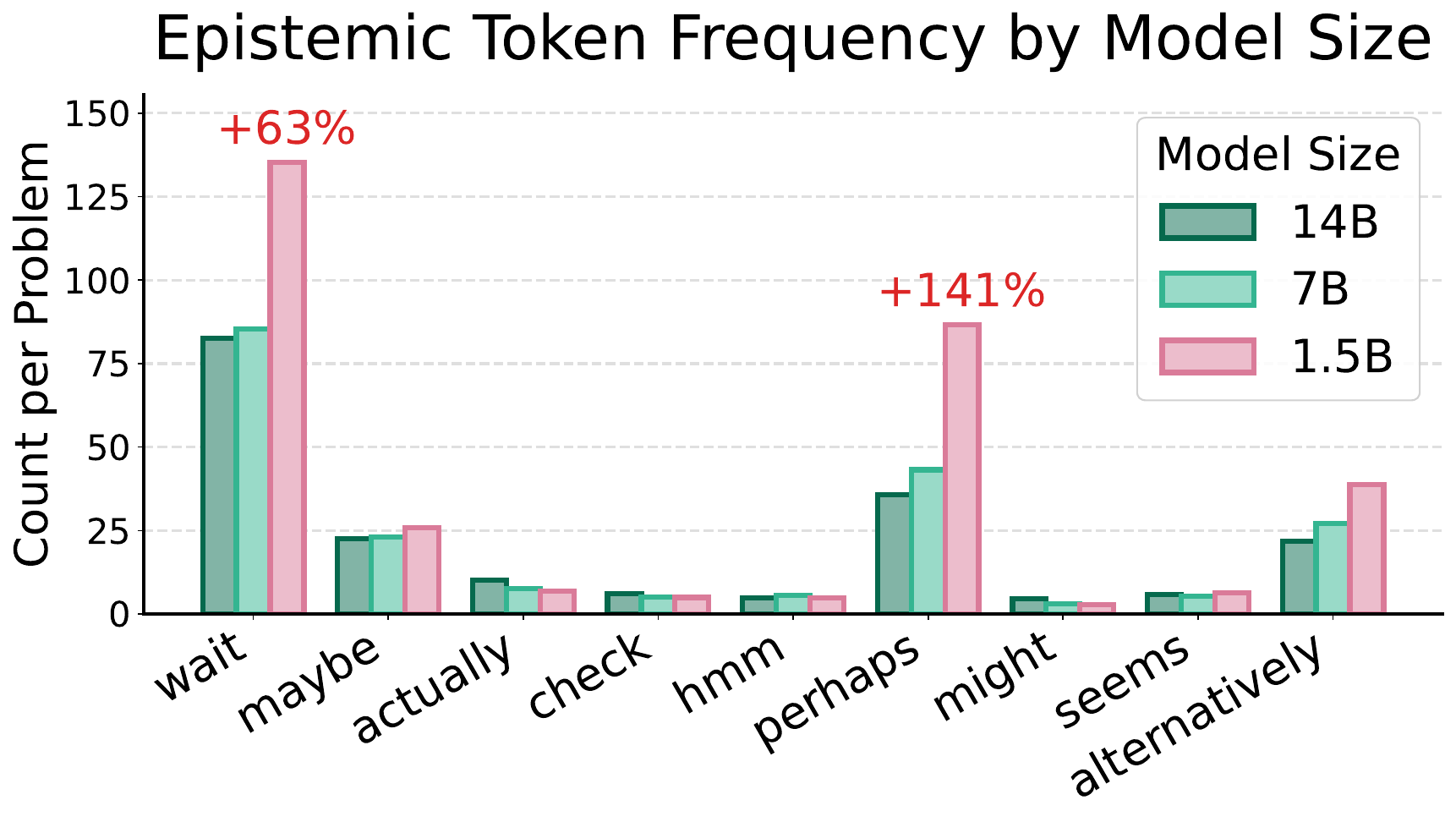}
    \includegraphics[width=0.49\linewidth]{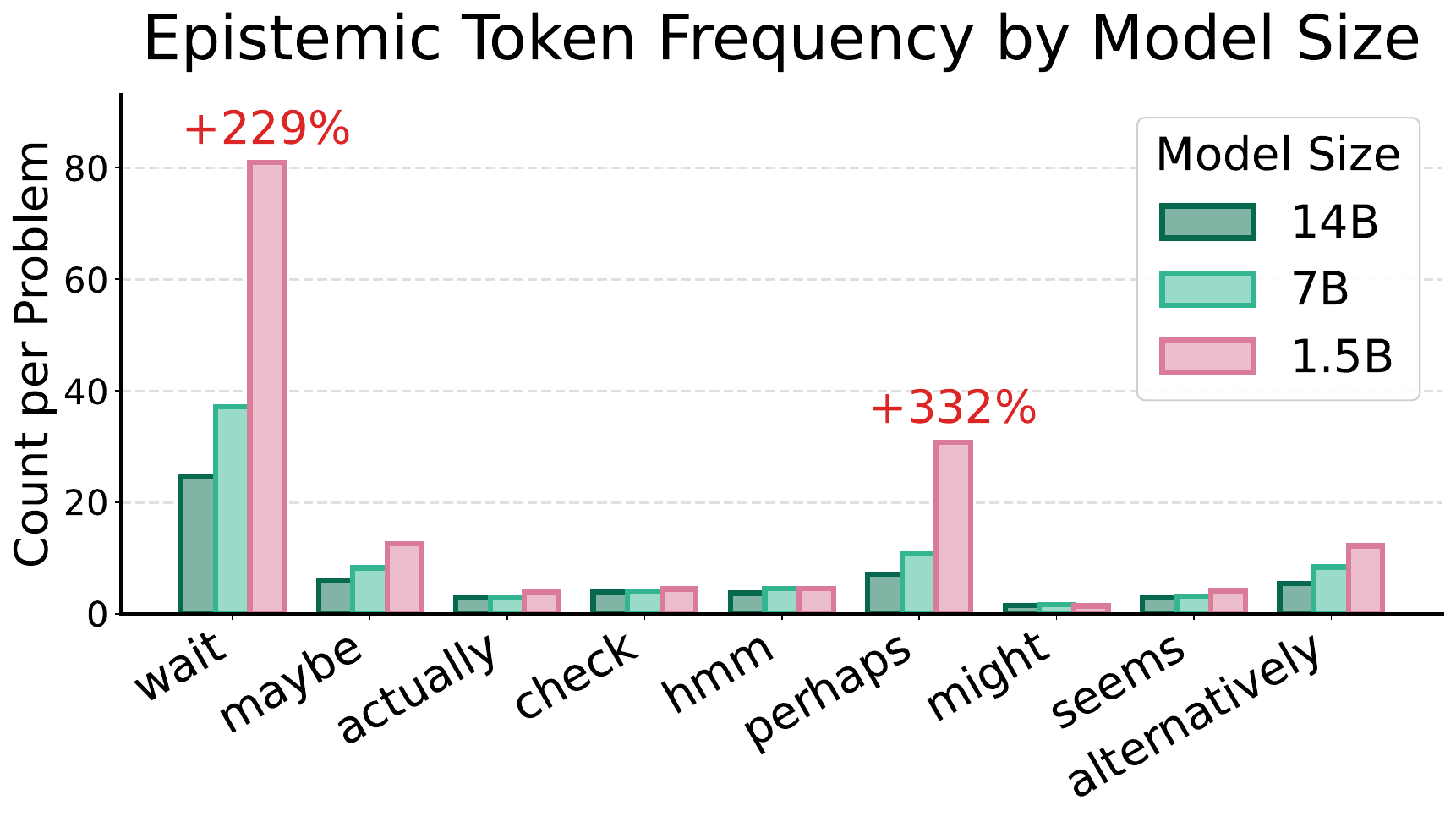}
    \includegraphics[width=0.49\linewidth]{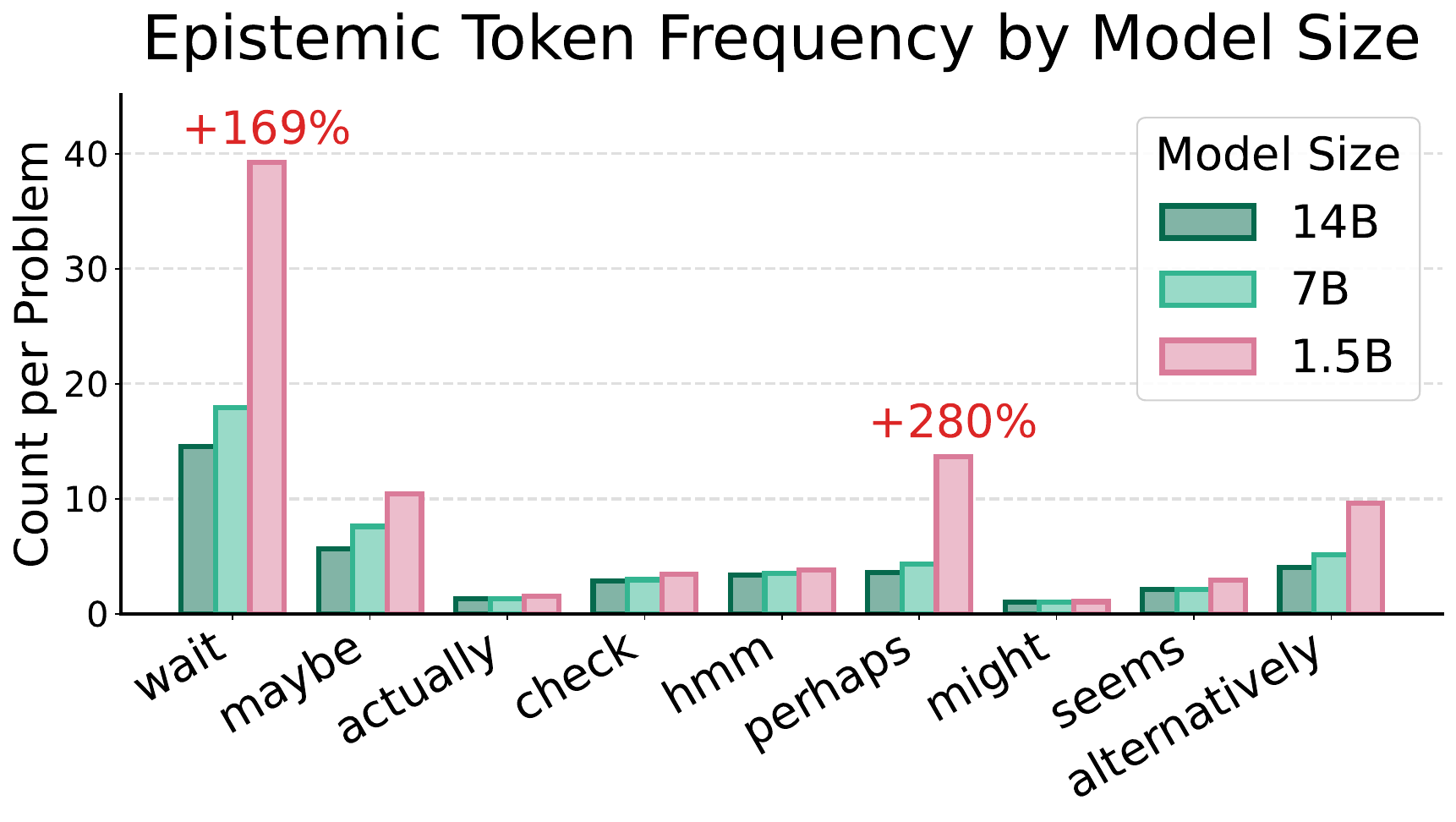}
    \caption{Token occurrence counts for DeepSeek-R1-Distill-Qwen-\{1.5B, 7B, 14B\} models on the AIME24, AIME25, AMC23, and MATH500 benchmarks.}
    \label{figure:uncertainty_epistemic_deepseek}
\end{figure}

\paragraph{Qwen3-Base LIMO.} Additionally, we investigated whether distilling a base model using a teacher model that heavily incorporates epistemic verbalization leads to increased generation of epistemic verbalizations depending on model size and problem difficulty. We observe that, in these distilled models as well, smaller models tend to generate more epistemic tokens as their uncertainty about a given problem increases.

Interestingly, unlike the analysis on DeepSeek-R1-Distill-Qwen models, where higher uncertainty substantially increased the use of the \textit{perhaps} token, these models do not show a marked increase in \textit{perhaps}. Instead, they primarily rely on the "Wait" token. As shown in Figure \ref{fig:limo_epistemic}, this may be because the LIMO dataset contains relatively few instances of perhaps. This suggests that patterns of uncertainty verbalization can shift depending on the training data distribution.

\begin{figure}[h!]
    \centering
   \includegraphics[width=0.49\linewidth]{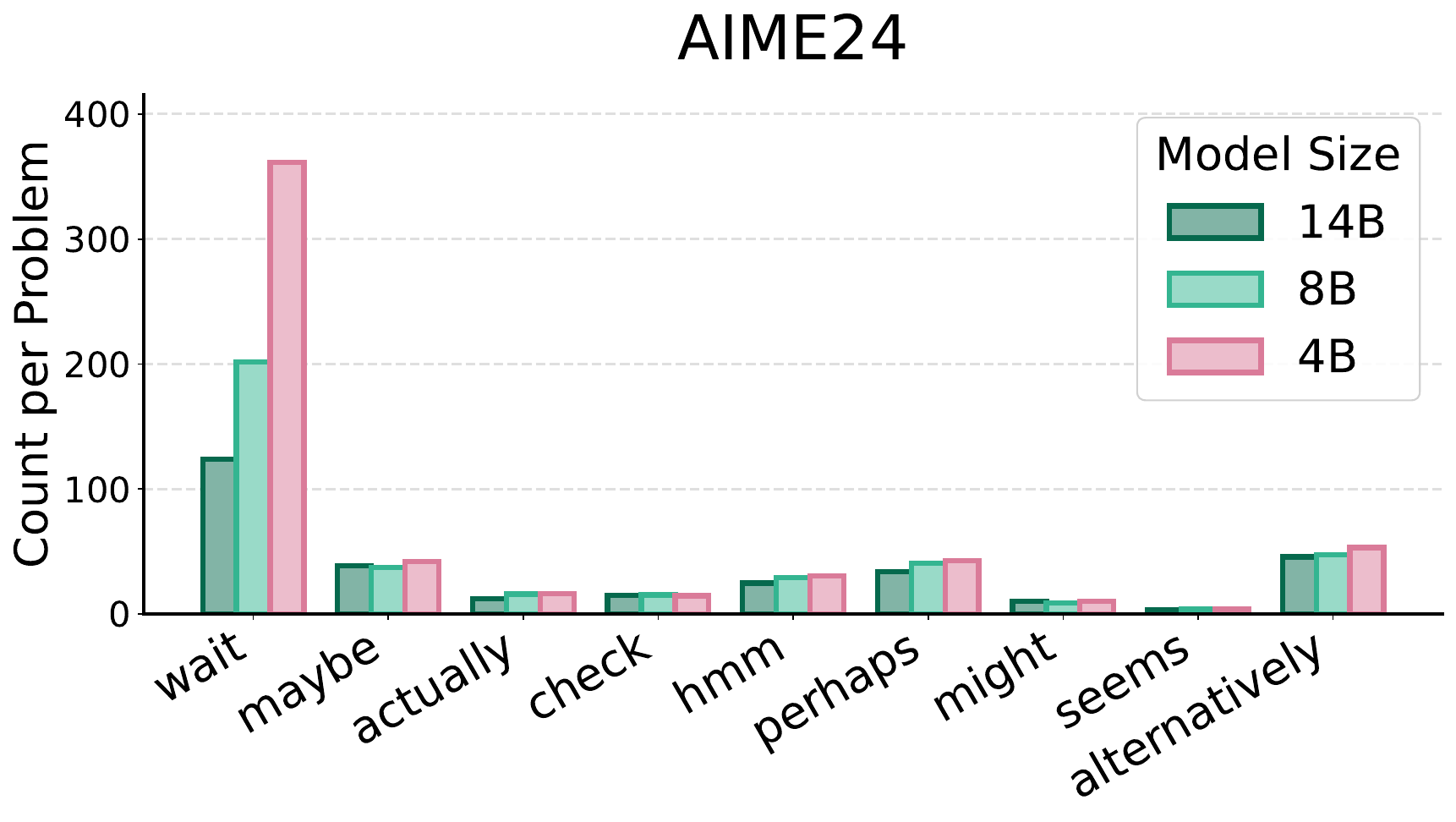}
    \includegraphics[width=0.49\linewidth]{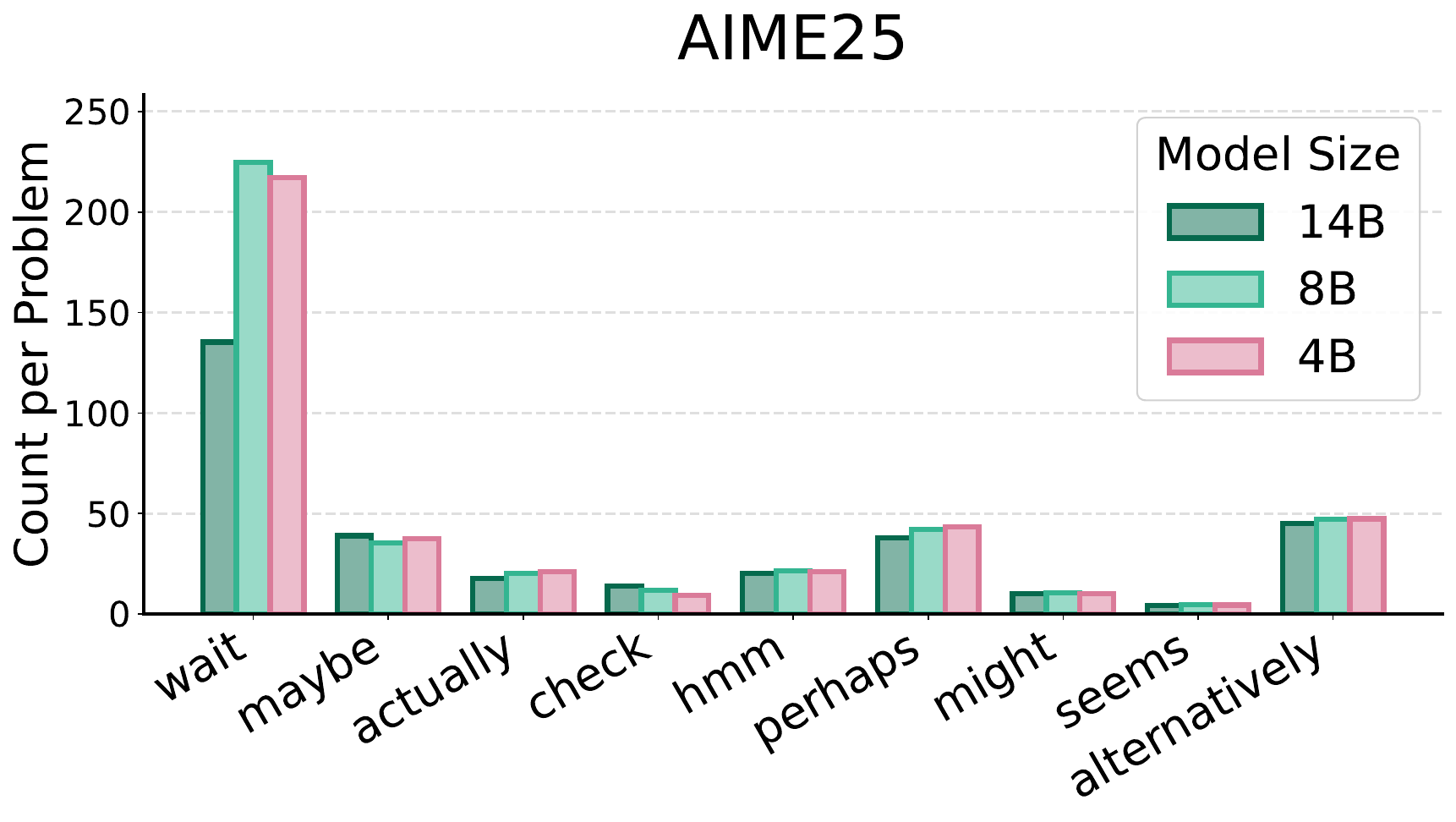}
    \includegraphics[width=0.49\linewidth]{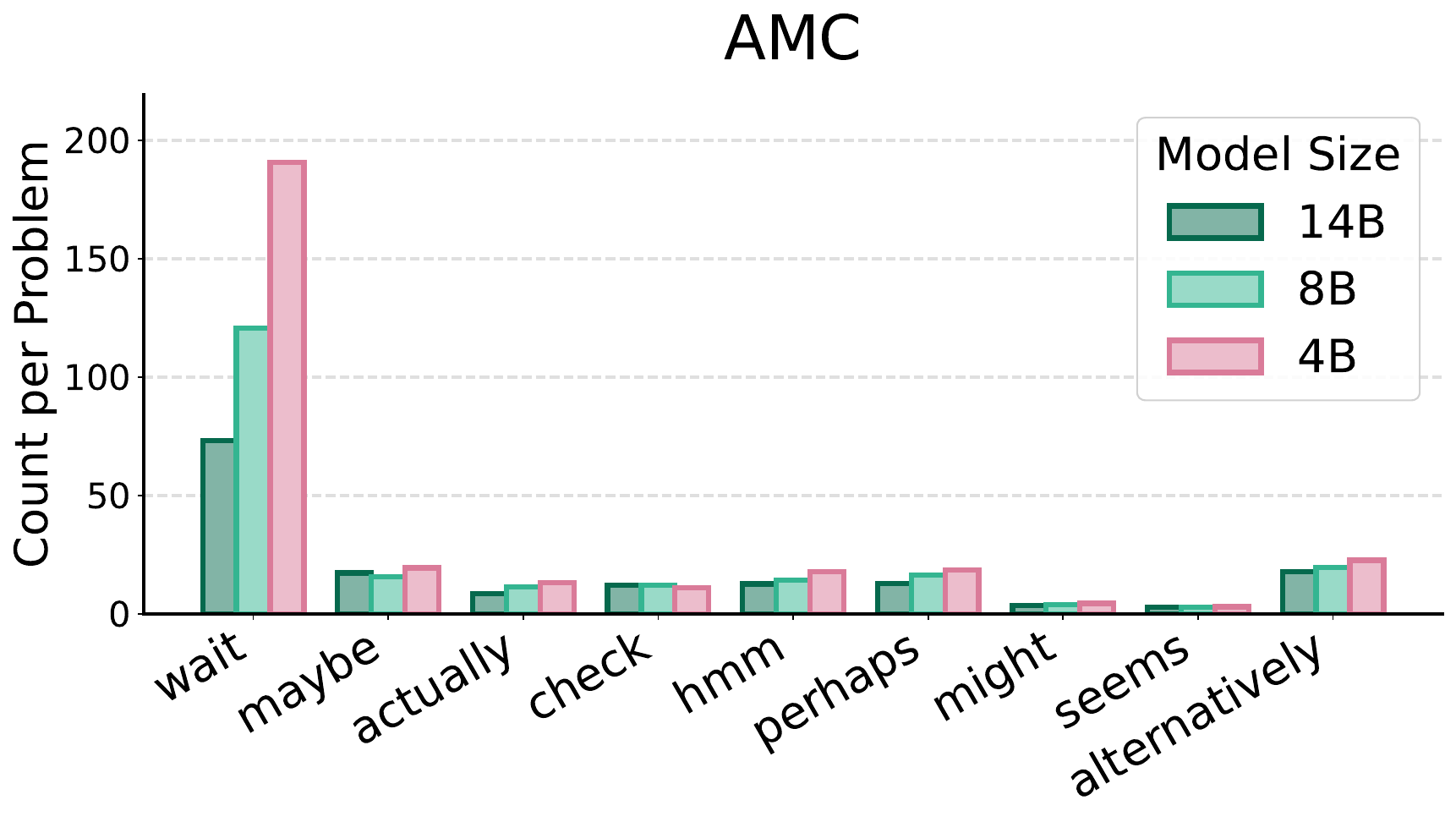}
    \includegraphics[width=0.49\linewidth]{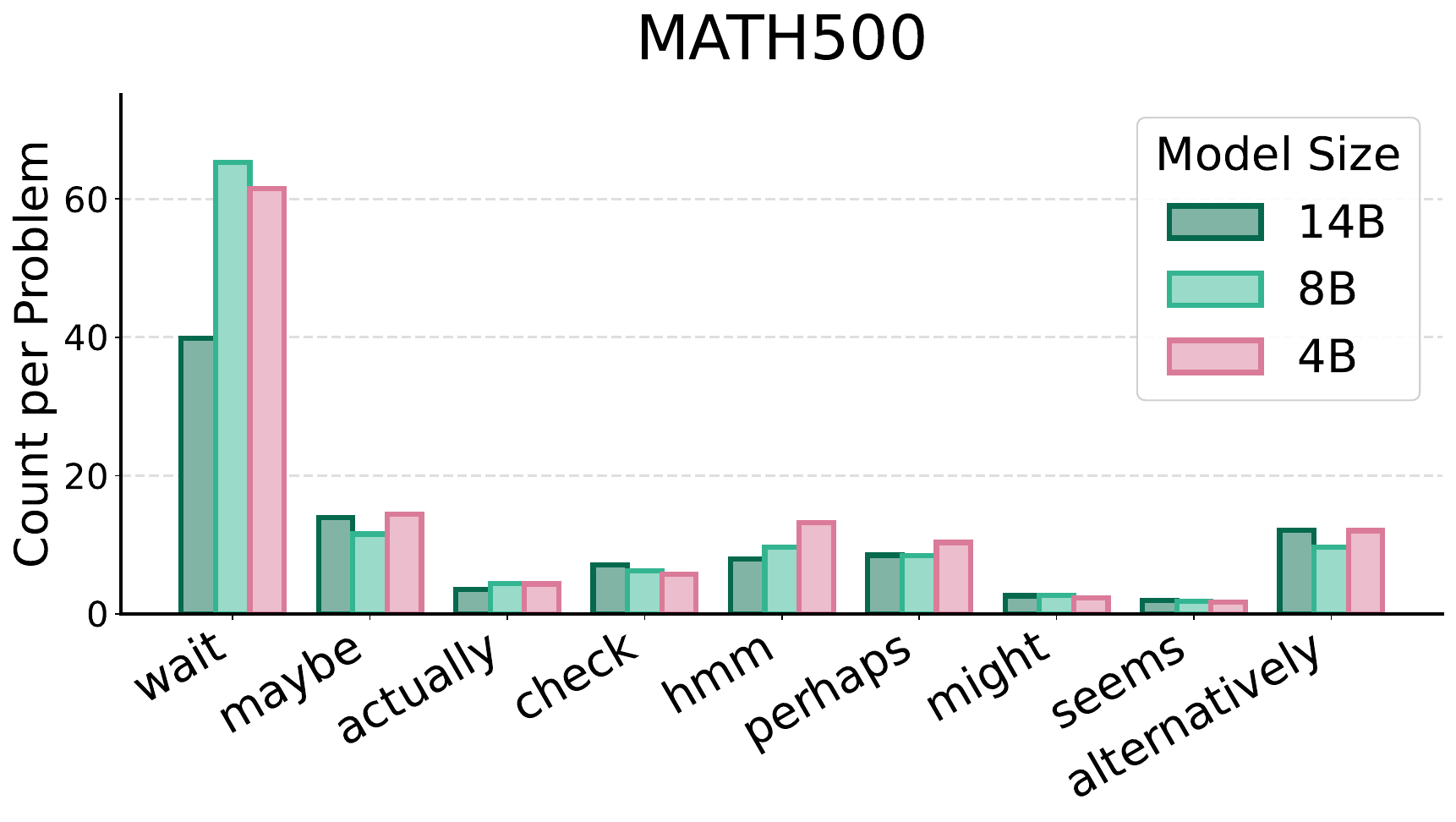}
    \caption{Token occurrence counts for Qwen3-\{4B, 8B, 14B\}-Base models on the AIME24, AIME25, AMC23, and MATH500 benchmarks.}
    \label{figure:uncertainty_epistemic_limo}
\end{figure}

\newpage
\section{Quantitative Analysis}
\subsection{Bypass Pathways under Epistemic Token Suppression} \label{appendix:bypassing_suppression}

In Section~\ref{section:test_time_control}, we suppress nine
epistemic tokens (\textit{wait}, \textit{hmm}, \textit{perhaps},
\textit{maybe}, \textit{actually}, \textit{alternatively},
\textit{seems}, \textit{might}, \textit{check}) via logit bias
$-100$ and observe performance drops of around 10\%. While
suppression substantially disrupts uncertainty expression,
performance does not collapse entirely. Analysis of the
reasoning traces reveals that the models sometimes route around
the masked tokens by substituting equivalents outside the banned
vocabulary, such as \textit{hold on}, \textit{no}, \textit{Let me},
or a paragraph break. Table~\ref{tab:suppression-bypass-patterns}
lists the observed substitutions. In these cases, the underlying doubt-and-verify pattern is preserved,
and only its lexical surface is rerouted.
\begin{table}[h!]
\centering
\small
\begin{tabular}{p{0.35\columnwidth} p{0.54\columnwidth}}
\toprule
\textbf{Baseline surface form} & \textbf{Suppression substitute} \\
\midrule
\textit{Wait, let me check\ldots}     & \textit{But hold on, let me\ldots} \\
\textit{Hmm,}                          & \textit{No, that's not quite right} / \textit{Hold on, no} \\
\textit{Maybe} / \textit{Perhaps}     & \textit{It's possible that\ldots} / \textit{Suppose\ldots} \\
\textit{Alternatively}                & \textit{Another way to see this} \\
\textit{Actually}                     & \textit{I realize} / \textit{On closer look} \\
\textit{Let me check}                 & \textit{Let me verify}  \\
\bottomrule
\end{tabular}
\caption{Bypass patterns observed when the epistemic tokens are suppressed during generation.}
\label{tab:suppression-bypass-patterns}
\end{table}

\subsection{Comparison Between Base and LIMO Distillation} \label{appendix:base_limo_trajectory}

\sethlcolor{custom_pink_light}

\begin{tcolorbox}[
colback=white,
title=\small{Qwen3-8B-Base AIME24 \#7},
colframe=custom_pink_dark,
coltitle=black,
colbacktitle=custom_pink,
breakable,
]
\small
Let's solve the problem step by step.

\textbf{Given:}

\begin{itemize}

    \item Let $\mathcal{B}$ be the set of rectangular boxes with surface area 54 and volume 23.

    \item Let $r$ be the radius of the smallest sphere that can contain each of the rectangular boxes that are elements of $\mathcal{B}$.

    \item We need to find $r^2$ in the form $\frac{p}{q}$, where $p$ and $q$ are relatively prime positive integers, and then compute $p + q$.

\end{itemize}
\textbf{Step 1: Define the Variables}
Let the dimensions of the rectangular box be $a$, $b$, and $c$. \\

\textbf{Step 2: Express the Given Conditions}

1. Surface Area: $2(ab + bc + ca) = 54 \implies ab + bc + ca = 27$

2. Volume: $abc = 23$  \\

\ldots\ldots\\

\textbf{Step 5: Use the AM-GM Inequality}

\hl{To minimize $a^2 + b^2 + c^2$, we can use the AM-GM inequality.} The minimum occurs when $a = b = c$.
Let $a = b = c$. Then:

$3a^2 = 27 \implies a^2 = 9 \implies a = 3$
$a^3 = 23 \implies 3^3 = 27 \neq 23$

This is not possible, so the minimum does not occur when $a = b = c$. \\

\textbf{Step 6: Use the Method of Lagrange Multipliers}

To find the minimum of $a^2 + b^2 + c^2$ subject to the constraints $ab + bc + ca = 27$ and $abc = 23$, we can use the method of Lagrange multipliers.

However, this method is complex for this problem. Instead, we can use the following approach. \\

\textbf{Step 7: Express $a^2 + b^2 + c^2$ in Terms of $ab + bc + ca$ and $abc$}

We know that: $a^2 + b^2 + c^2 = (a + b + c)^2 - 2(ab + bc + ca)$

But we don't have $a + b + c$. Instead, we can use the identity: $a^2 + b^2 + c^2 = (a + b + c)^2 - 2(ab + bc + ca)$

However, without $a + b + c$, we need another approach. \\

\hl{\textbf{Step 8: Use the Identity for $a^2 + b^2 + c^2$}}

We can use the identity: $a^2 + b^2 + c^2 = (a + b + c)^2 - 2(ab + bc + ca)$
But we don't have $a + b + c$. Instead, we can use the identity: $a^2 + b^2 + c^2 = (a + b + c)^2 - 2(ab + bc + ca)$

However, without $a + b + c$, we need another approach. \\

\ldots\ldots\\

\hl{\textbf{Step 244: Use the Identity for $a^2 + b^2 + c^2$}}

We can use the identity: $a^2 + b^2 + c^2 = (a + b + c)^2 - 2(ab + bc + ca)$

But we don't have $a + b + c$. Instead, we can use the identity: $a^2 + b^2 + c^2 = (a + b + c)^2 - 2(ab + bc + ca)$

However, without $a + b + c$, we need another approach. \\

\hl{\textbf{Step 245: Use the Identity for $a^2 + b^2 + c^2$}}

We can use the identity: $a^2 + b^2 + c^2 = (a + b + c)^2 - 2(ab + bc + ca)$

But we don't have $a + b + c$. Instead, we can use the identity: $a^2 + b^2 + c^2 = (a + b + c)^2 - 2(ab + bc + ca)$

\ldots\\

\end{tcolorbox}

\begin{tcolorbox}[
colback=white,
title=\small{LIMO Qwen3-8B-Base AIME24 \#7},
colframe=custom_green_dark,
coltitle=black,
colbacktitle=custom_green,
breakable,
]
\small
Okay, so I need to find the smallest sphere that can contain any rectangular box from the set $\mathcal{B}$, which has surface area 54 and volume 23. Then, I have to find $r^2$ where $r$ is the radius of that sphere, express it as a fraction $\frac{p}{q}$ in lowest terms, and compute $p + q$. \\

First, let me recall that the smallest sphere that can contain a rectangular box is the one whose diameter is equal to the space diagonal of the box. The space diagonal of a rectangular box with dimensions $a, b,$ and $c$ is $\sqrt{a^2 + b^2 + c^2}$. Therefore, the diameter of the sphere is equal to this space diagonal, so the radius $r$ would be half of that, right? So, $r = \frac{1}{2}\sqrt{a^2 + b^2 + c^2}$. Therefore, $r^2$ would be $\frac{a^2 + b^2 + c^2}{4}$.  \\

Therefore, to find the minimal $r^2$, I need to minimize $a^2 + b^2 + c^2$ over all possible boxes with surface area 54 and volume 23. So the problem reduces to minimizing $a^2 + b^2 + c^2$ given that $ab + bc + ca = 27$ (since surface area is $2(ab + bc + ca) = 54$) and $abc = 23$ (since volume is $abc = 23$).  \\

\hlpink{Hmm, so I have to minimize $a^2 + b^2 + c^2$ given $ab + bc + ca = 27$ and $abc = 23$.}  \\

I remember that in optimization problems with constraints, Lagrange multipliers might be helpful, but since these are symmetric constraints, maybe I can use some inequalities here. Alternatively, since it's about rectangular boxes, maybe the minimal space diagonal occurs when the box is a cube? Wait, but maybe not necessarily a cube, but perhaps when the dimensions are in some proportion? \\

Solving quadratic equation: $a = \frac{-2 \pm \sqrt{4 + 92}}{2} = \frac{-2 \pm \sqrt{96}}{2} = \frac{-2 \pm 4\sqrt{6}}{2} = -1 \pm 2\sqrt{6}$. \\

Since $a$ is a length, it must be positive. So, possible solutions are $a = 2$, or $a = -1 + 2\sqrt{6} \approx 3.899$, or $a = -1 - 2\sqrt{6}$ which is negative, so discard that. \\

So, if $a = 2$, then $c = 23 / 4 = 5.75$.
Check if this satisfies the first equation: $a^2 + 2ac = 4 + 2 \cdot 2 \cdot 5.75 = 4 + 23 = 27$. Yes, that works. So, in this case, the dimensions are $a = b = 2, c = 5.75$. \\

\ldots\ldots\\

Now, let's compute $a^2 + b^2 + c^2 = 4 + 4 + (23/4)^2 = 8 + 529/16 = (128 + 529)/16 = 657/16 \approx 41.0625$. \\

Alternatively, if we take the other root $a = -1 + 2\sqrt{6} \approx 3.899$, then $c = 23 / a^2 \approx 1.513$. Then, let's compute $a^2 + b^2 + c^2 \approx 15.204 + 15.204 + 2.289 \approx 32.697$. \hlgreen{Hmm, which is larger than 657/16... wait, $32.697$ is smaller than $41.0625$.} \\

\hlgreen{Wait, wait a second! I think I made a critical mistake here.}  \\

The problem asks for the smallest sphere that can contain each of the rectangular boxes that are elements of $\mathcal{B}$. Since $\mathcal{B}$ is the set of all boxes with surface area 54 and volume 23, the minimal sphere must work for all such boxes, so we need the maximum possible space diagonal among all such boxes, then the radius is half that. \hlgreen{Therefore, actually, we need to maximize the space diagonal, not minimize it.} \\

\hlgreen{Wait, so my earlier approach was incorrect.} I thought we needed to minimize the space diagonal, but actually, since the sphere must contain all boxes, the radius must be at least half the maximum space diagonal of any box in $\mathcal{B}$. Therefore, the minimal such sphere's radius $r$ is determined by the maximum of $a^2 + b^2 + c^2$. \\

Using $(a + b + c)^2 = a^2 + b^2 + c^2 + 2(ab + bc + ca) = (a^2 + b^2 + c^2) + 54$.
So, to maximize $a^2 + b^2 + c^2$, we must maximize $a + b + c$. \\

From my symmetric cases analysis:
Case 1 ($a=2$): $a+b+c = 2+2+5.75 = 9.75$.
Case 2 ($a \approx 3.899$): $a+b+c \approx 3.899+3.899+1.513 = 9.311$. \\

The maximum occurs at $a=2, b=2, c=5.75$.
Max $a^2 + b^2 + c^2 = 657/16$. \\

Therefore, the minimal radius $r$ of the sphere that can contain \textit{any} box in $\mathcal{B}$ satisfies:
$r^2 = \frac{1}{4} \cdot \max(a^2 + b^2 + c^2) = \frac{1}{4} \cdot \frac{657}{16} = \frac{657}{64}$. \\

Checking if $\frac{657}{64}$ is in lowest terms. $657 = 3^2 \cdot 73$ and $64 = 2^6$. No common factors.
$p = 657, q = 64$.
$p + q = 657 + 64 = 721$.
\end{tcolorbox}

\section{Examples of Reactive and Proactive Correction}
\label{appendix:example_two_correction_mode}

Table~\ref{tab:correction_examples} lists representative examples of the two correction modes discussed in Sections~\ref{sec:collapse_escape} and~\ref{sec:proactive}, pairing each model's preceding computation with the verbalization that triggers correction. \textit{Reactive} corrections, observed in both LLMs and LRMs, are triggered by an explicit failure in the preceding computation (e.g., a violated constraint, a contradiction between derivations, or a hallucinated runtime error). \textit{Proactive} corrections, observed only in LRMs, arise without any such surfaced error.

\section{Use of AI Assistants}
\label{appendix:ai_assistants}

We used Claude and ChatGPT to polish and refine text drafted by the authors, to create Figure~\ref{fig:concept}, and to improve the design of plots. As described in the main text, LLMs including GPT-5 were used as automated judges and evaluation tools as part of the experimental pipeline. No AI assistants were used to generate research ideas or design experiments.

\begin{table*}[!t]
\centering
\small
\begin{tabular}{p{3.0cm}p{5.9cm}p{5.9cm}}
\toprule
\textbf{Model} & \textbf{Preceding computation} & \textbf{Trigger verbalization} \\
\midrule
\multicolumn{3}{l}{\textit{LLM --- reactive correction only}} \\[3pt]

Qwen2.5-32B-Instruct
  & Diagonal midpoint equation yields $x = 6$, which explicitly violates the given constraint $x > 7$.
  & \textit{``However, we are given that $x > 7$. Therefore, we need to re-evaluate the vertices.''} \\[6pt]

Qwen2.5-7B
  & Substitution gives $y = \tfrac{5}{2z}\big|_{z=5/2} = 1$, violating the condition $y > 1$.
  & \textit{``We made an error in the interpretation of the equations.''} \\[6pt]

Qwen2.5-Math-7B
  & \texttt{solve(...)} raises \texttt{TypeError: list indices must be integers or slices, not Symbol}.
  & \textit{``It seems there was an issue with how the solution was being accessed.''} \\

\midrule
\multicolumn{3}{l}{\textit{LRM --- reactive correction}} \\[3pt]

DeepSeek-R1-Distill-Qwen-32B
  & $AP = \dfrac{2025 - 4225}{286} = \dfrac{-2200}{286} < 0$
  & \textit{``Wait, negative length? That can't be.''} \\[6pt]

DeepSeek-R1-Distill-Qwen-32B
  & Cayley--Menger: $V = \frac{20\sqrt{2}}{9}$;\quad coordinate method: $V = \frac{160}{3}$
  & \textit{``That contradicts the earlier volume \ldots\ So, clearly, something is wrong.''} \\[6pt]

Qwen3-14B
  & Numerical evaluation shows $f(2)\approx4.16$, $f(5)\approx4.83$, $f(10)\approx4.61$: non-monotone, contradicting the earlier claim of monotonic increase.
  & \textit{``Wait, so it increases up to some point and then decreases? That contradicts my previous thought.''} \\

\midrule
\multicolumn{3}{l}{\textit{LRM --- proactive correction}} \\[3pt]

DeepSeek-R1-Distill-Qwen-32B
  & Serially recalling symmedian length formulas. \textbf{(each is immediately replaced by another with no computation performed.)}
  & \textit{``Wait, perhaps I should look for a different approach.''} \\[6pt]

Qwen3-14B
  & Game-tree analysis of $n=2$ is proceeding correctly. \textbf{(no error has surfaced)}.
  & \textit{``Wait, maybe I need to think again.''} \\[6pt]

DeepSeek-R1-Distill-Qwen-32B
  & Correctly identifying $n=2$ as a losing position for the moving player. \textbf{(no contradiction present)}.
  & \textit{``Wait, maybe I'm getting confused.''} \\

\bottomrule
\end{tabular}
\caption{Representative examples of reactive and proactive correction in LLMs and LRMs.
  LLMs exhibit predominantly reactive correction, triggered by explicit computational failures or contradiction. LRMs additionally exhibit proactive correction, where doubt arises without any surfaced error.}
\label{tab:correction_examples}
\end{table*}

\end{document}